\documentclass[preprint]{elsarticle}
\pdfoutput=1
\usepackage{amsmath,amssymb}
\usepackage{appendix}
\usepackage{amsthm}
\usepackage{graphicx}
\usepackage{subfigure}
\usepackage{float}
\usepackage{cases}
\usepackage{color}
\usepackage{xcolor}
\makeatletter
\def\ps@pprintTitle{%
	\let\@oddhead\@empty
	\let\@evenhead\@empty
	\def\@oddfoot{}%
	\let\@evenfoot\@oddfoot}
\makeatother

\newtheorem{definition}{Definition}%[section]
\newtheorem{remark}{Remark}%[section]
\newtheorem{lemma}{Lemma}%[section]
\newtheorem{assumption}{Assumption}%[section]
\usepackage{algorithm}
\usepackage{algorithmic}

\begin{document}
	\begin{frontmatter}
		\title{Safe Learning-based Gradient-free Model Predictive Control Based on Cross-entropy Method}
		
		\author[label1]{Lei Zheng\fnref{fn1}}
		\author[label2]{Rui Yang}
		\author[label2]{Zhixuan Wu}
		\author[label2]{Jiesen Pan}
		\author[label2]{and Hui Cheng\corref{correspondingauthor}}
		\ead{chengh9@mail.sysu.edu.cn}
		\address[label1]{School of Electronics and Information Technology, Sun Yat-sen University, Guangzhou 510006, China}
		\address[label2]{School of Computer Science and Engineering, Sun Yat-sen University, Guangzhou 510006, China}
		\fntext[fn1]{First author. Email: zhenglei5@mail2.sysu.edu.cn}
		\cortext[correspondingauthor]{Corresponding author}
		
		\begin{abstract}
			In this paper, a safe and learning-based control framework for model predictive control (MPC) is proposed to optimize nonlinear systems with a {non-differentiable} objective function under uncertain environmental disturbances. The control framework integrates a learning-based MPC with an auxiliary controller in a way of minimal intervention. The learning-based MPC augments the prior nominal model with incremental Gaussian Processes to learn the uncertain disturbances. The cross-entropy method (CEM) is utilized as the sampling-based optimizer for the MPC with a {non-differentiable} objective function. A minimal intervention controller is devised with a control Lyapunov function and a control barrier function to guide the sampling process and endow the system with high probabilistic safety. The proposed algorithm shows a safe and adaptive control performance on a simulated quadrotor in the tasks of trajectory tracking and obstacle avoidance under uncertain wind disturbances.
		\end{abstract}
		
	\end{frontmatter}
	
	\section{Introduction}
	
	Robots and autonomous systems are increasingly widely applied to solve complex tasks in highly uncertain and dynamic environments \cite{siciliano2016springer}. Operating in such environments requires sophisticated control methods that can adapt to the uncertain environmental disturbances, plan and execute trajectories utilizing the full dynamics and complete the predefined tasks safely. Model predictive control (MPC) \cite{mayne2014model} provides a general framework to consider the system constraints naturally, anticipate future events and take control actions accordingly to complete complex tasks which are encoded in the objective function. However, designing differentiable objective functions corresponding to all requirements of the task is non-trivial \cite{schaal2010learning}. For example, it is challenging to design a differentiable objective function accounting for simultaneously avoiding unexpectedly detected obstacles and aggressively tracking trajectory for a safety-critical quadrotor in a clustered obstacle field. The trade-off between safety and tracking control performance is hard to mediate using a simplified differentiable form. It is more convenient for the designer to compose multiple high-level, simple but possibly {non-differentiable} terms in the objective function, but the resulting {non-differentiable} objective function brings difficulties to the optimization \cite{williams2018robust}. Obtaining the solutions to the optimization problem using regular gradient-based optimizers {is} difficult due to the nonlinear dynamics constraints and {non-differentiable} objective function.
	
	{
    There are some optimization methods such as nonconvex alternating direction method of multipliers (ADMM)\cite{wang2019global}\cite{o2013splitting}, sequential quadratic programming, and interior point methods \cite{wright1997primal}\cite{andersson2019casadi} for optimizing the system with nonconvex or even sparse non-differentiable cost function. However, either linear system or derivative information is required, making it difficult to directly apply to the optimization with nonlinear dynamics constraints and arbitrary non-differentiable objective function.} Sampling-based methods such as random shooting \cite{nagabandi2018neural}, path integral control \cite{williams2017model}, and cross-entropy method (CEM) \cite{de2005tutorial} are effective approaches to solve the optimization problem with a {non-differentiable} objective function.
	The CEM is initially devised to estimate the probability of a rare event and later employed as a general optimization framework.
	{The optimal sampling distribution in CEM is recursively approximated to guide the sampled trajectories toward lower cost regions until converging to a delta distribution\cite{de2005tutorial}.
	As with most randomized methods, there is no guarantee that this would be the global optimum but with high likelihood, the local optimum can be approached if enough samples are generated.} It has been demonstrated to effectively optimize complex nonlinear systems with a {non-differentiable} objective function \cite{kobilarov2012cross}\cite{chua2018deep}.
	An adaptive sampling strategy is developed in \cite{tan2018gaussian} using receding-horizon cross-entropy trajectory optimization with Gaussian Process upper confidence bound \cite{srinivas2009gaussian}. In \cite{bharadhwaj2020model}, the planning problem of high-dimensional systems is solved by interleaving CEM and gradient descent optimization. Actually, as a local search method, the CEM is essentially prone to the model errors caused by unexpected disturbances, which may make the sampling get stuck in a bad region of state space and even diverge, resulting in dangers to safety-critical systems \cite{williams2018robust}. To address this problem, decent initialization and gradient signal are required in this case to lead sampling distribution back to the low-cost region safely. {To illustrate, it is common for robotics systems that the actual next state is close to the predicted one in a high-frequency control scheme. Thus it is reasonable to initialize the sampling distribution of next iteration with information from the last iteration as a type of warm start.} On the other hand, the accuracy of the predictive model under uncertain disturbances should be improved.
	
	\begin{figure}
		\centering
		\includegraphics[width=10cm]{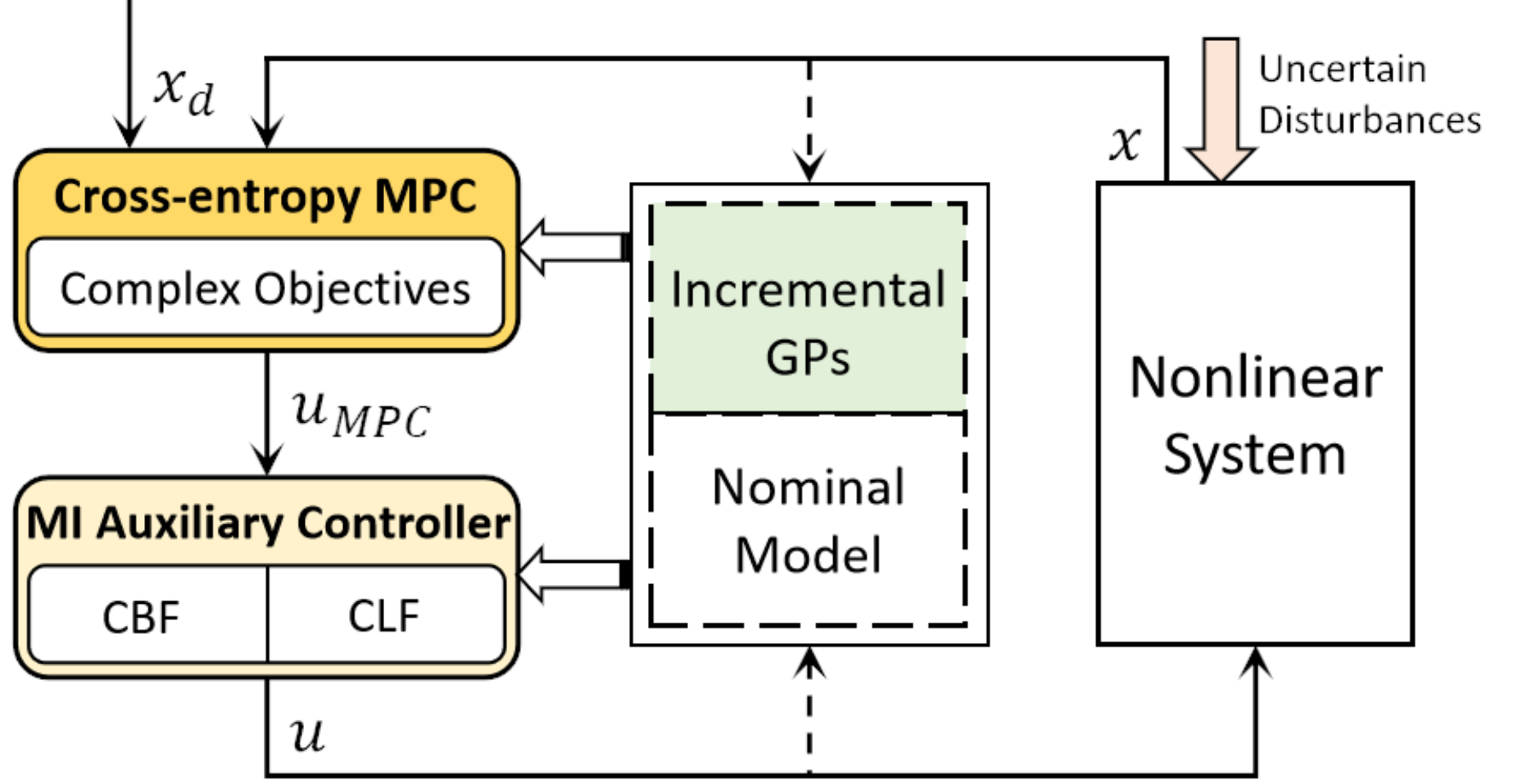}\\
		\caption{Diagram of the proposed control scheme for a nonlinear system under uncertain disturbances. Taking in the desired state $x_d$ and current state $x$, the MPC optimizes the nonlinear system with a simple {non-differentiable} objective function using CEM. The output $u_{MPC}$ of MPC is modified by an auxiliary controller in a way of minimal intervention, which guides the sampling process and preserves system safety. The environmental uncertainties are learned using incremental GPs. The learned model is combined with a prior nominal model to serve in both the CEMPC and the MI{(Minimal Intervention)} auxiliary controller.}\label{fig:structure}
	\end{figure}
	To guide the sampling distribution safely in case of external disturbances while keeping the control proactive, a sampling-based Tube-MPC method is augmented in \cite{williams2018robust} with an iterative linear quadratic Gaussian controller for sampling distribution guidance and disturbance rejection. Model predictive path integral control \cite{williams2017model} is integrated with $\mathcal{L} 1$ adaptive control in \cite{pravitra2020l1} to achieve both fast model predictive trajectory planning and robust trajectory tracking. Besides, the minimum intervention principle has shown an increasing interest in the safe control domain of semi-autonomous vehicles. A minimal intervention (MI) mechanism is designed in \cite{leung2020infusing} to infuse reachability-based safety assurance within the planning frameworks. It essentially projects the desired trajectory into a set of safety-preserving controls whenever safety is threatened, so as to ensure safety and meanwhile keep the behavior as proactive as possible. Similarly, the control barrier function (CBF) is utilized to design the safety barrier certificates in \cite{Luo2020MultiRobotCA} \cite{Hirshberg2020SafetyCI} to minimally modify an existing controller to formally satisfy collision avoidance constraints via a constrained quadratic program (QP). {Multiple obstacles situations can be handled based on CBF \cite{wang2018permissive} \cite{deits2015computing}.} While robust control under bounded uncertainties achieves safety and meanwhile preserves the control performance of the existing controller, the adaptation to disturbances is not considered in \cite{Luo2020MultiRobotCA} \cite{Hirshberg2020SafetyCI}. The CBF and control Lyapunov function (CLF) are integrated into a QP to achieve both safety and tracking stability in \cite{zheng2020learning}, but without a minimal intervention scheme.

	On the other hand, the control performance of MPC depends on the accuracy of the predictive model \cite{mayne2014model}. However, it is generally difficult to obtain precise system models for real-world nonlinear systems. This issue of model discrepancy has been addressed through various adaptation approaches. Some of the existing adaptive MPC approaches assume a structured system model with uncertain parameters that can be estimated online with an estimator such as extended Kalman filter \cite{fukushima2007adaptive}\cite{aswani2013provably}. However, it is limited to treat all model errors as parameters to estimate, especially when the system is subject to complex and external uncertain disturbances \cite{desaraju2016experience}.

	%	 In \cite{Hentzen2019DisturbanceEA}\cite{bouffard2012learning}, external aerodynamic effects of quadrotor are estimated using extended Kalman filter and compensated in the predictive model of MPC, leading to an improved accuracy of position control. 
	
	Considering model uncertainties resulting from uncertain disturbances, learning-based techniques can also be applied to learn the predictive model in MPC. In the model-based reinforcement learning (RL), the dynamics model is learned from data using a deep neural network (NN) \cite{nagabandi2018neural} or Gaussian Processes (GPs) \cite{urbina2011quantification}\cite{cao2017gaussian}. The effectiveness of CEM for model-based RL is demonstrated in \cite{chua2018deep} with Bayesian neural network ensemble as the predictive model. While avoiding the need {for} manual controller design, these works do not well utilize the prior knowledge on systems and control design \cite{kober2013reinforcement}, which can be used to improve the learning efficiency and provide safety for system design. Taking advantage of prior knowledge of the system, semi-structured approaches augment a prior nominal model with machine learning models to capture the model errors for MPC \cite{hewing2020learning}\cite{desaraju2016experience}. GPs are utilized to learn the model errors, since {they} can simultaneously capture the uncertainty of the estimation due to the lack of data and the noises inherent in the environment \cite{urbina2011quantification}. It is combined with a nominal model as the predictive model for MPC \cite{hewing2018cautious}\cite{ostafew2016robust}\cite{mehndiratta2020gaussian}, {as a kind of robust or stochastic MPC\cite{mayne2016robust}}, where the control performance is shown {to greatly improve} with a more accurate predictive model. While the safety regarding state constraints can be satisfied in the form of chance constraint \cite{hewing2020learning}, {the propagation of state variance along the prediction horizon brings much computation} and the optimization of the resulting nonlinear MPC requires gradient information of the objective function.

	%	In Neural-Lander \cite{shi2019neural}, a deep neural network with spectral normalization is used to learn the disturbances of quadrotor. It is combined with a nominal model to design a learning-based feedback linearization controller with a stability guarantee. To acquire the proactive ability and robustness inherent in the receding horizon optimization of the MPC, 

	%	The computational complexity of GP regression strongly depends on the number of data
	%	points m, which motivates the use of sparse approximations
	%	using inducing points 

	The limitations of the sampling-based MPC and the advantages of safe learning-based control techniques show an urgent need for designing an adaptive, high-performance and safe control strategy for nonlinear system optimization with a {non-differentiable} objective function. Safety should be enforced to the sampling-based optimization while the proactive control ability should be maintained for the nonlinear system under environmental disturbances. To solve this problem, we propose a safe and adaptive predictive control architecture as shown in Fig.\ref{fig:structure}.
	
	%  The nominal controller, termed the CEMPC controller, attempts to solve the primary optimal control problem for a nonlinear system based on a predictive model augmented by incremental GPs (IGPs), and the second controller, called the MI ancillary controller consisting of the CBF and CLF, has the goal of enforcing safety constraints and guiding the sampling distribution back towards the low-cost region.
	
	Our main contributions are summarized as follows.
	\begin{itemize}
		\item A novel learning-based CEM-based MPC (CEMPC) framework is proposed for system optimization with a {non-differentiable} objective function. The CEM is utilized as the optimizer to solve the {non-differentiable} MPC based on a prior predictive model and incremental GPs (IGPs) for {additive state} disturbance estimation.
		\item A minimal intervention (MI) auxiliary controller based on CBF and CLF is devised to intervene {in} the sampling-based MPC, endowing the system with safety and {guiding} the sampling distribution to low-cost regions when necessary. 
		%		It intervenes the MPC by the minimum amount necessary to achieve the desired result.
		\item The proposed control methodology is validated on a quadrotor tracking aggressive trajectory and simultaneously avoiding detected obstacles under uncertain time-varying wind disturbances in simulation.
	\end{itemize}
	
	The rest of this paper is organized as follows. 
	The problem statement is presented in Section~\ref{section:problem_statement}.
	The IGPs, learning-based CEMPC and the MI auxiliary controller are described in Section~\ref{section:methodology}.
	Numerical simulation of the proposed algorithm on a quadrotor system is shown in Section~\ref{section:experiment}. Finally, a conclusion is drawn in Section~\ref{section:conclusion}.

	\section{Problem Statement}
	\label{section:problem_statement}
	Consider a nonlinear control affine system with dynamics
	\begin{equation}
	\dot{x}=f(x) + G(x)u,
	\label{eq:system_model}
	\end{equation}
	where $x \in \mathcal{X} \subset \mathbb{R}^{n}$ denotes the system state and $u \in \mathcal{U} \subset \mathbb{R}^{m}$ is the control input.
	{A wide range of robots such as quadrotors and car-like vehicles can be transformed into a control affine system in this form. Though our analysis is restricted to this form, the results can be extended to the systems of higher relative degrees.}
	Assume that the function $f:\mathcal{X}\rightarrow\mathbb{R}^n$ is partially unknown but {Lipschitz continuous and} has a bounded reproducing kernel Hilbert space (RKHS) norm under a known kernel, and the function $G:X\rightarrow\mathbb{R}^{n\times {m}}$ is known and Lipschitz continuous. The partially unknown $f(x)$ consists of a known nominal model $\hat{f}(x)$ and the uncertain disturbances $d(x)$,
	\begin{equation}
	\label{eq:model_error}
	f(x)=\hat{f}(x) + d(x).
	\end{equation}
	If there is no disturbance, then the nominal model matches the actual one and $d(x)=0$. However, it is difficult to get an accurate model in advance for practical nonlinear systems under uncertain disturbances, e.g. for a quadrotor under uncertain wind disturbances.
	
% 	{
% 	\begin{remark}
% 	    Note that the prior model $\dot{x}=\hat{f}(x) + G(x)u$ is a model derived by physically and mathematically modeling in the first principle. The structure and parameters of the model can be determined or identified according to prior knowledge of the system, which is usually not exactly matched to the actual ones. Considering a complex nonlinear system, in this sense, an identified linear  model can also be used as the prior model, as long as the system satisfies the assumptions of the method in the paper. Even though the identified linear model is not well performing, the system can improve the control performance via safe online learning.
% 	\end{remark}
% 	}
	
	The \textbf{goal} is to optimize the system (\ref{eq:system_model}) to accomplish specified complex tasks safely under uncertain environmental disturbances. The specified tasks can be described with the \textbf{cost function} of the form:
	\begin{equation}
	\mathcal{L}(x,u)=\iota(x)^{T}Q\iota(x)+\sum_{i=1}^{N}w_i\mathbb{I}_{C_i}(x)+ u^TR_uu,
	\label{eq:cost_function}
	\end{equation}
	where $\iota(x)$ extracts features from the state, $Q$ and $R_u$ are a {positive semidefinite} and a positive definite weight matrix, respectively. $N$ is the number of simple encodings of task descriptions, $w_i$ is the corresponding weight coefficient, $\mathbb{I}_{C_i}$ is an indicator function for the set $C_i$, which is $1$ if $x\in C_i$ and $0$ otherwise. 
	
	%simple and sparse encodings of task descriptions 
	
	The first portion of the cost function (\ref{eq:cost_function}) is to encode the main task for the system, e.g. tracking a trajectory. The second term encodes specific tasks, which could be sparse, {non-differentiable} and hard to rewrite into a typically differentiable form like the first term, e.g. to avoid the unexpectedly detected obstacles for a quadrotor with a limited sensing range. The last term regularizes the control inputs. With this form of the cost function, the tasks can be easily encoded with several interpretable terms and weighted differently according to the importance of task requirements in convenience.
	\begin{remark}
		{Note that though the non-differentiable terms are technically soft constraints, the penalty is obtained immediately once the state falls into the specified sets. Besides, they have the advantage that the importance of different requirements can be delineated by setting different weight coefficients $w_i$.}
	\end{remark}

	%%%%%%%%%%%%%%%%%%%%%%%%%%%%%%%%%%%%%%%%%%%%%%%%%%%%%%%%%%%%%%%%%%%%%%%%%%%%%%%%
	\section{Methodology}
	\label{section:methodology}
	In this section, the proposed control scheme for the nonlinear system (\ref{eq:system_model}) is described, as shown in Fig.~\ref{fig:structure}.
	With a predefined objective function consisting of simple and {non-differentiable} encodings of task descriptions, the CEMPC optimizes the nonlinear system with a predictive model. This predictive model is composed of a nominal model from prior knowledge and a discrepancy model learned via GPs. To guide the sampling distribution back toward the low-cost regions in case of disturbances and preserve system safety, the control inputs computed by the CEMPC are modified with a designed MI auxiliary controller in an efficient QP framework. 
	
	The GPs for learning the uncertain disturbances ($\ref{eq:model_error}$) and the incremental implementation to reduce the computational complexity are first introduced in Section~\ref{section:GP}. With the prior and learned model, the CEMPC is then presented in Section~\ref{subsection:learning-based CEMPC} and the MI auxiliary controller is described in the Section~\ref{subsection:MI}.

	\subsection{Increment Gaussian Processes for Disturbance Learning}
	\label{section:GP}
	A GP is an efficiently nonparametric regression method to estimate complex functions and their uncertain distribution~\cite{rasmussen2010gaussian}. It assumes that function values associated with different inputs are random variables, and any finite number of them have a joint Gaussian distribution.	The uncertain model errors $d(x)$ resulted from uncertain disturbances in $(\ref{eq:model_error})$ can be learned using GPs with the data collected from the system during operation. Similar to \cite{ostafew2016robust}, we train $n$ separate GPs to model the disturbances $d(x)$ with the output of $n$ dimensions based on the {following} assumption.  
	\begin{assumption}
		The unknown disturbances $d(x)$ in (\ref{eq:model_error}) are uncorrelated.
		\label{uncorrelated assumption}
	\end{assumption}
	The approximation of disturbances $d(x)$ can be denoted by $\widetilde d$. To make the problem tractable, similar to the Assumption 1 in~\cite{Berkenkamp2016SafeLO}, the following assumption is considered. 
	\begin{assumption}
		The unknown disturbances $d(x)$ has a bounded norm in the associated Reproducing Kernel Hilbert Space (RKHS)~\cite{scholkopf2002learning}, corresponding to a continuously differentiable kernel $k$. 
		\label{regularity assumption}
	\end{assumption}
	%	In this study, a incremental GP is designed to learn the model uncertainty $d(x)$ in (\ref{eq:system_model}) with the regularity \textbf{Assumption 1}.
	This assumption can be interpreted as a requirement {for} the smoothness of the disturbances $d(x)$. Its boundedness implies that $d(x)$ is regular with respect to the kernel $k$ \cite{srinivas2012information}. It is common practice to use the squared-exponential kernel
	$k(x,\ x^\prime)=\sigma_f^2\exp{(-\frac{1}{2}(x-x^\prime)^\mathrm{T}L^{-2}(x-x^\prime))}$ to estimate the similarity between states $x$ and $x\prime$, which is characterized by hyperparameters of the length scale diagonal matrix $L$ and the prior variance $\sigma_f^2$. Besides, we assume that the training set $D$ is available to the GPs for regression.
	\begin{assumption}
		The {predefined} state $x$ and the function value $d(x)$ can be measured with noises over a finite time horizon to make up a training set with $n$ data pairs
	\end{assumption}
	\begin{equation}
	D=\left\{\left(x^{\left(i\right)},\ y^{\left(i\right)}\right)\right\}_{i=1}^n,\ y^{(i)}=d\left(x^{(i)}\right)+\upsilon_i,
	\end{equation}
	where $\upsilon_i$ are i.i.d. noises $\upsilon_i\sim \mathcal{N}\left(0,\sigma_{noise}^2I_n\right)$, $\sigma_{noise}^2\in\mathbb{R}$. 
	
	Given the dataset $D$, the mean and variance of $\widetilde d(x_{*})$ at the query state $x_{*}$ can be given by \cite{rasmussen2010gaussian}
	\begin{equation}
	\mu(x_{*})=k_{n}^{T}(K_\sigma+\sigma_{noise}^{2}I)^{-1}y_{n},
	\label{mean}
	\end{equation}
	\begin{equation}
	\sigma^{2}(x_{*})=k(x_{*},x_{*})-k_{n}^{T}(K_\sigma+\sigma_{noise}^{2}I)^{-1}k_{n},
	\label{var}
	\end{equation}
	respectively, where $y_{n}=[y(x_{1}),y(x_{2}),...,y(x_{n})]^{T}$ is the observed vector, $K_\sigma\in\mathbb{R}^{n\times n}$ is the covariance matrix with entries $[K_\sigma]_{(i,j)}=k(x_{i},x_{j})$, $i,j \in \left\{ 1, ..., n \right\}$, $k_{n}=[k(x_{1},x_{*}),k(x_{2},x_{*}),...,k(x_{n},x_{*})]$ is the vector with kernel function $k(x_{i},x_{j})$, $I_n\in\mathbb{R}^{n\times n}$ is the identity matrix. The variance of noise {$\sigma_{noise}$ can be selected and set according to the actual noise level of the disturbances and the accuracy of measurements.}
	
	A high probability confidence interval $\mathcal D(x)$ on $\widetilde d(x)$ can then be obtained~\cite{Berkenkamp2016SafeLO}
	\begin{equation}
	\mathcal{D}(x)=\{\widetilde d\ |\ \mu(x)-c_{\delta}\sigma(x) \leq \widetilde d \leq \mu(x)+c_{\delta}\sigma(x)\},
	\label{high confidence interval}
	\end{equation}
	where $c_{\delta}$ is a parameter designed to get a confidence interval of $(1 - \delta)$, $\delta \in (0, 1)$.
	For instance, $95.5\%$ and $99.7\%$ confidence of the uncertainty bound can be achieved at $c_{\delta}=2$ and $c_{\delta}=3$, respectively.	
	
	The computational complexity of GPs is $O(n^3)$ due to the matrix inverse of $K$, $K = K_\sigma+\sigma^{2}I$. It brings non-negligible challenges for practical applications. An incremental implementation of GPs is devised to reduce the time of inference and learning by fixing the size of the dataset and recursively computing the matrix inverse. Specifically, the IGP considers two stages as follows.
	
	%Given a maximum fixed-size kernel matrix $\mathbf{k}$, we add the latest data point into the kernel matrix and delete the oldest data point simultaneously at each timestep. 
	
	\emph{1) Adding New Data:} Given a predefined size $n$ of the dataset, a new data point is added into the dataset at each time step until the size of the dataset reaches the predefined size. Assume there has been $i$ data points in the dataset, $0<i<n$. Denote the computed kernel matrix as $K_{old} \in \mathbb{R}^{i\times i}$ and the corresponding matrix inverse as $K^{-1}_{old}$ based on the old dataset. {When a new data point $(x_{i+1},y_{i+1})$ is added}, the new kernel matrix $K_{new} \in \mathbb{R}^{(i+1) \times (i+1)}$ can be computed based on $K_{old}$: 
	\begin{equation}
	K_{new}=\begin{bmatrix}K_{old} & \mathbf{k}_{i+1}   \\ \mathbf{k}^T_{i+1} & k_{i+1}+\sigma_{noise}^2I\end{bmatrix},
	\end{equation}
	where {$\mathbf{k}_{i+1}=[k(x_{1},x_{i+1}), k(x_{2},x_{i+1}), \cdots, k(x_{i},x_{i+1})]^\intercal$ and} $k_{i+1} = k(x_{i+1}, x_{i+1})$. The incremental update of the matrix inverse $K_{new}^{-1}$ can be computed by
	\begin{equation}
	K_{new}^{-1}=\begin{bmatrix}K_{old}^{-1}+\xi\Gamma\cdot\Gamma^T & -\xi\Gamma \\ -\xi\Gamma^T & \xi\end{bmatrix}
	\end{equation}
	where $\Gamma=K_{old}^{-1}\cdot \mathbf{k}_{i+1}$ and $\xi=(k_{i+1}+\sigma_{noise}^2I-\mathbf{k}^T_{i+1}\cdot \Gamma)^{-1} \in \mathbb{R}^{1\times 1}$.
	
	\begin{remark}
		The inversion operation only needs to perform on a scalar $\xi$ rather than on the entire matrix $K_{new} \in \mathbb{R}^{(i+1) \times (i+1)}$. Matrix multiplication in the method takes $O(i^2)$ complexity. Thus, this method can reduce the computation burden largely by saving and reusing the computed matrix inverse result, especially when the predefined size $n$ of the dataset is quite large. 
	\end{remark}

	\emph{2) Replacing Data:} If the size of the dataset reaches the predefined size $n$, a new data point will be added to the dataset at each timestep while the oldest one will be deleted. Denote the old kernel matrix in the form of a block matrix as
	\begin{equation}
	K_{old}=\begin{bmatrix}k_0 & \mathbf{k}^T_0 \\ \mathbf{k}_0 & \Upsilon\end{bmatrix},
	\end{equation}
	where $\mathbf{k}_0$ and $k_0$ are the covariance vector and variance value of the oldest data point in the dataset, and $\Upsilon \in \mathbb{R}^{(n-1) \times (n-1)}$ are the sub-matrix at the right bottom corner. The inverse matrix of $K_{old}$ can be computed as 
	\begin{equation}
	K_{old}^{-1}=\begin{bmatrix}\rho & q^T \\ q & \Xi \end{bmatrix},
	\end{equation}
	where $\rho$, $q$ and $\Xi$ is the {sub-matrices}, $\Xi \in \mathbb{R}^{(n-1) \times (n-1)}$. With the obtained new data point, the oldest data point is removed from the dataset, and variance $k_{i+1}$ and covariance vector $\mathbf{k}_{i+1}$ are computed. The new kernel matrix can be obtained based on the old kernel matrix $K_{old}$:
	\begin{equation}
	K_{new}=\begin{bmatrix}\Upsilon &\mathbf{k}_{i+1} \\ \mathbf{k}^T_{i+1} & k_{i+1}+\sigma_{noise}^2I\end{bmatrix},
	\end{equation}
	and the inverse matrix can be computed by 
	\begin{equation}
	K_{new}^{-1}=\begin{bmatrix}\Lambda+(\Lambda\mathbf{k}_{i+1})(\Lambda \mathbf{k}^T_{i+1})l & -(\Lambda \mathbf{k}_{i+1})l \\ -(\Lambda\mathbf{k}^T_{i+1})l & l\end{bmatrix},
	\end{equation}
	where $\Lambda=\Xi-q\cdot q^T\rho^{-1}$, $l=(k_{i+1}+\sigma_{noise}^2I-\mathbf{k}_{i+1}^T\cdot \Lambda\cdot \mathbf{k}_{i+1})^{-1}$.

	%	$\Lambda=\rho\Xi-q\cdot q^T$
	\begin{remark}
		Note that there are only 4 times of matrix multiplications in this method. Taking several matrix addition and transposition into consideration, it takes $O(n^2)$ computational complexity.
	\end{remark}

	%%%%%%%%%%%%%%%%%%%%%%%%%%%%%%%%%%%%%%%%%%%%%%%%%%%%%%%%%%%%%%%%%%%%%%%%%%%%%%%%
	\subsection{Learning-based Model Predictive Control with Cross-Entropy Method (CEMPC)} 
	
	\label{subsection:learning-based CEMPC}
	With the learned disturbances $\widetilde{d}$ via the IGPs, the predictive model for the MPC can be obtained based on the prior model. Considering the {non-differentiable} objective function, a sampling-based MPC scheme is designed with CEM to provide the nominal controls for the nonlinear systems (\ref{eq:system_model}). The MPC is based on the following open-loop finite horizon optimal control problem given the measured state $x(t_k)$ at each sampling time $t_k=k\cdot ${$T_s$} with a control period {$T_s$}, $k\in \mathbb{N^+}$.
	\begin{alignat}{2}
	u^{*}=\mathop{\arg\min}_{\bar{u}(t)\in \mathcal{PC}([t_k, t_k+T], {\mathcal{U}})}& \  \Phi(\bar{x}(t+T)) + \int_{t_k}^{t_k+T}\mathcal{L}(\bar{x}(\tau),\bar{u}(\tau))d\tau., \label{eq:opt1}\\
	\mbox{s.t.} \quad
	&\dot{\bar{x}}(t) =\hat{f}(\bar{x}(t))+ g(\bar{x}(t))\bar{u}(t) + \widetilde d(\bar{x}(t)), \label{eq:opt2}\\
	&\bar{x}(t_k)=x(t_k),
	\label{eq:opt3}\end{alignat}
	%	\bar{u}(t)&\in \mathcal{U}, \label{eq:opt4}\\
	%	\bar{x}(t)&\in \mathcal{X},	\label{eq:opt5}
	where $\mathcal{PC}([t_k, t_k+T], {\mathcal{U}})$ represents the set of all piece-wise continuous functions $\varphi: [t_k, t_k+T]\rightarrow{\mathcal{U}}$, $\bar{x}$ denotes the state predicted based on the system model (\ref{eq:opt2}) given candidate controls $\bar{u}(t)$ over the prediction horizon $T>0$, $\Phi$ is a terminal cost function, $\mathcal{L}$ is the {non-differentiable} running cost (\ref{eq:cost_function}), and $\widetilde d$ is the approximation to the actual environmental uncertainties $d(x)$. The equation ($\ref{eq:opt3}$) is the state initialization of the finite horizon optimal control problem.
	{The input constraints are taken into account by bounding the control input samples in the control space and limiting the initial mean and variance of the sampling distribution.}
	
	Optimization and execution take place alternatively. With the current state $x(t_k)$ {fedback} at any sampling time $t_k$, the problem (\ref{eq:opt1})-(\ref{eq:opt3}) is solved to obtain the optimal open-loop control inputs $u^{*}(t),\forall t \in [t_k, t_k+T]$, and only the first-step input $u_{MPC}(t)=u^{*}(t),\forall t \in [t_k, t_k+${$T_s$}$]$ is applied to the nonlinear system (\ref{eq:system_model}). The overall process is repeated at the next sampling time $t_{k+1}$.

	\begin{algorithm}[htb] 
		\caption{Learning-based CEMPC} 
		\label{alg:Framwork_1} 
		\begin{algorithmic}[1] % number 1 represents display index
			\REQUIRE ~~\\ %Input
			$N$: Number of iterations,\\
			$M$: Sample numbers per iteration,\\
			$H$: Predictive time steps of the MPC,\\
			$K$: Size of the elite set,\\
			$\Sigma_{min}$: A minimum variance bound for optimization,\\
			$\beta$: Update rate,\\
			$\mathcal{N}(O_{0:H-1}^{(0)}, \Sigma_{0:H-1}^{(0)})$: Initial sampling distribution, where $O_{0:H-1}^{(0)}$ and $ \Sigma_{0:H-1}^{(0)}$ denotes the mean and covariance matrix of $H$ separate initial multivariate Gaussian distributions, respectively. \\
			\ENSURE ~~\\ %Output
			$\mu_{*}$ : Optimized mean value of the control input sampling distribution,\\
			\WHILE{Task is not completed}
			\STATE Measure current state $x_{k}$.
			\STATE $i=0$.
			\WHILE{$i<N$ and $max(\Sigma_{0:H-1}) > \Sigma_{min}$} 
			\STATE Sample $\{(u_0^{(i)},...,u_{H-1}^{(i)})_j\}_{j=0}^{M-1}\sim\mathcal{N}(O_{0:H-1}^{(i)}, \Sigma_{0:H-1}^{(i)})$.
			\STATE Score the samples according to (\ref{eq:score_function}), (\ref{eq:opt2}) and (\ref{eq:opt3}).
			\STATE Sort them in an ascending order, $\mathcal{J}_0^{(i)} \leq ... \leq \mathcal{J}_{M-1}^{(i)}$.
			\STATE Choose $K$ sequences according to $\mathcal{J}_0^{(i)} \leq ... \leq \mathcal{J}_{K-1}^{(i)}$.
			\STATE Update sampling distribution using the elite set \\
			$O_{0:H-1}^{(i+1)} \gets Mean(\{(u_0^{(i)},...,u_{H-1}^{(i)})_j\}_{j=0}^{K-1}),$ \\
			$\Sigma_{0:H-1}^{(i+1)} \gets Var(\{(u_0^{(i)},...,u_{H-1}^{(i)})_j\}_{j=0}^{K-1})$.
			\STATE  $(O_{0:H-1}^{(i+1)},\Sigma_{0:H-1}^{(i+1)})${$\leftarrow$}$(1-\beta)(O_{0:H-1}^{(i+1)},\Sigma_{0:H-1}^{(i+1)})+ \beta(O_{0:H-1}^{(i)},\Sigma_{0:H-1}^{(i)})$
			\STATE {$i \leftarrow i+1$}
			\ENDWHILE 
			\STATE $u_{MPC}${$\leftarrow$}$\{(u_{0}^{(N-1)})_j\}_{j=0}$.
			\STATE $u_k$ $\leftarrow$ MIControlSheme$(x_k, u_{MPC})$ in \textbf{Algorithm \ref{alg:Framwork_2}}.
			\STATE $x_{k+1}$ $\leftarrow$ Apply $u_k$ to the system (\ref{eq:system_model}).
			\STATE Collect data point $\{x_k,u_k,x_{k+1}\}$ and Update $GP$ in (\ref{mean}) and (\ref{var}).
			\STATE Reinitialize sampling distribution $O_{0:H-2}^{(0)}\gets O_{1:H}^{(N-1)}$, 
			$\Sigma_{0:H-2}^{(0)} \gets \Sigma_{1:H}^{(N-1)}$.
			\ENDWHILE 
		\end{algorithmic}
	\end{algorithm}

	It is hard to obtain a closed-form solution or apply a gradient-based optimizer to the optimization problem (\ref{eq:opt1})-(\ref{eq:opt3}), since the dynamics (\ref{eq:opt2}) is {nonlinear}, and the objective function (\ref{eq:opt1}) is {non-differentiable}. The CEM is adopted as an adaptive sampling-based method to solve the MPC, which is widely applied as a general optimization framework to solve complex optimization problems\cite{finn2017deep}\cite{kobilarov2012cross}\cite{chua2018deep}. It treats the optimization problem as an estimation problem of the probability of a rare event, with the distribution parameters to be estimated. The control space is sampled repeatedly and the sampling distributions of the controls are optimized via the importance sampling technique.

	To ease the understanding, the optimization process with CEM for (\ref{eq:opt1})-(\ref{eq:opt3}) is described in discrete control. Denote the time variable $(t+k\cdot ${$T_s$}$)$ with the time step subscript $k$. At the $i$-th CEM iteration, multiple $M$ random control sequences with $H$ time steps are sampled
	\begin{equation}
	\label{eq:sequences}
	\{(u_0^{(i)},...,u_{H-1}^{(i)})_j\}_{j=0}^{M-1}\sim\mathcal{N}(O_{0:H-1}^{(i)}, \Sigma_{0:H-1}^{(i)}),
	\end{equation}
	where $\mathcal{N}(O_{0:H}^{(i)}, \Sigma_{0:H}^{(i)})$ denotes $H$ separate multivariate Gaussian distributions, from which control input sequences are sampled. The mean and covariance matrix of the distributions are set initially $O_{0:H}^{(0)}=\{0\}_{0:H}$ and $\Sigma_{0:H}^{(0)}=\{\Sigma_{init}\}_{0:H}$, respectively, where $\Sigma_{init}=(\frac{u_{max}-u_{min}}{2})^2$. With these control sequences, the accumulated costs $\mathcal{J}_j^{(i)}$ of the $j$-th control sequences are evaluated based on the cost function $\mathcal{L}$ (\ref{eq:cost_function}): 
	\vspace{-3mm}
	\begin{equation}
	\mathcal{J}_j^{(i)}=\Phi(x_H) + \sum_{k=0}^{H-1}\mathcal{L}(x_k,u_k), \forall \ j=0,...,M-1. 
	\label{eq:score_function}
	\end{equation}
	With the estimated costs $\{J_j^{(i)}\}_{j=0}^{M-1}$, an elite set of $K$ control sequences with the $K$ lowest costs are chosen out from the $M$ sequences (\ref{eq:sequences}). The elite set is used to update the sampling distribution $\mathcal{N}(O_{0:H}^{(i+1)}, \Sigma_{0:H}^{(i+1)})$ for the next $i+1$ iteration: 
	\begin{alignat}{2}
	O_{0:H}^{(i+1)} &\gets (1-\beta)Mean(\{(u_0^{(i)},...,u_{H-1}^{(i)})_j\}_{j=0}^{K-1}) + \beta 	O_{0:H}^{(i)}, \\
	\Sigma_{0:H}^{(i+1)} &\gets (1-\beta) Var(\{(u_0^{(i)},...,u_{H-1}^{(i)})_j\}_{j=0}^{K-1}) + \beta 	\Sigma_{0:H}^{(i)}.
	\end{alignat}
	where $\beta$ is a smoothness coefficient to adjust the update of the distribution parameters. {The smaller it is, the more we trust the new distribution parameters. We usually set its value by experience and test different values in practice.}
	
	The sampling distribution is updated towards the regions with lower cost, as the above process iterates for $N$ times or the maximum variance drops below a minimum variance bound $\Sigma_{min}$. The remaining control sequence $\{(u_0^{(N)},...,u_{H-1}^{(N)})_j\}_{j=0}$  is returned with the lowest cost and the first-step control input $u_{MPC}=\{(u_{0}^{(N)})_j\}_{j=0}$ is applied as the output of the CEMPC. Modified by the auxiliary controller when necessary, the control input is applied to the system. The data of the system evolution is then collected in the dataset to update the GPs. \textbf{Algorithm 1} details the proposed learning-based CEMPC. 
	\begin{remark}
		Note that the prediction of mean in (\ref{mean}) takes $O(n^2H+MnH)$ computational complexity when predicting $M$ samples over H timesteps in (\ref{eq:sequences}).
		The calculation of the cost function in (\ref{eq:score_function}) is applied for $M$ samples and $H$ timesteps for each sample, which takes $O(MH)$ complexity.
	\end{remark}
	%%%%%%%%%%%%%%%%%%%%%%%%
	\subsection{Minimal Intervention (MI) Auxiliary Controller}
	\label{subsection:MI}
	In the learning-based CEMPC, task requirements and state constraints can be encoded in the cost functions with different penalty weights. Though the design process is simplified in this way, system safety regarding state constraints is not guaranteed. Besides, the sampling distribution should be guided to the low-cost regions in the case of unexpected disturbances. To solve the problem, an auxiliary controller is designed in this part to minimally intervene {in} the CEMPC.

	\subsubsection{Barrier-enforced Safety Scheme}
	\label{section:CBF}
	
	Safety under environmental uncertainties should be considered carefully before the optimized control inputs are applied to the system. It can be enforced with safety constraints via CBF, which can quantify the system safety regarding state constraints~\cite{ames2019control}.
	
	The $\mathit{safety\ set}\ \mathcal{S}$ of the system (\ref{eq:system_model}) can be defined by
	\begin{equation}
	\mathcal{S}:=\{x\in\mathcal{X}|h(x)\geq0\},
	\label{safety set}
	\end{equation}
	where $h:\mathbb{R}^{n}\rightarrow\mathbb{R}$ is a continuously differentiable function related to the state constraints.
	
	\begin{definition}
		The set $\mathcal{S}$ is called $\mathit{forward\ invariant}$, if for every $x_{0}\in\mathcal{S}$, $x(t,x_{0})\in\mathcal{S}$ for all $t\in \mathbb{R}_{0}^{+}$.
	\end{definition}
	To ensure forward invariance of $\mathcal{S}$, e.g. quadrotors stay in the collision-free safety set at all times, we consider the following definition.
	
	\begin{definition}(Definition 5 of \cite{ames2017control})\label{cbf_definition} 
		For the dynamical system~(\ref{eq:system_model}), given a set $\mathcal{S}\subset{\mathbb{R}}^{n}$ defined by~(\ref{safety set}) for a continuously differentiable function $h:\mathbb{R}^{n}\rightarrow\mathbb{R}$, the function $h$ is called a $\mathit{Zeroing \ Control\  Barrier \  Function}$ $(ZCBF)$ defined on the set $\mathcal{E}$ with $\mathcal{S} \subseteq \mathcal{E} \subset{\mathbb{R}}^{n}$, if there exists an extended class $\mathcal{K}$ function $\kappa$ ($\kappa(0) = 0$ and strictly increasing) such that
		\begin{equation}
		\mathop{\sup}_{u \in \mathcal{U}}[L_{f}h(x) + L_{g}h(x)u + \kappa (h(x))]\geq0, \forall x\in\mathcal{E},
		\end{equation}
		where $L$ represents the Lie derivatives. 
	\end{definition}
	To be more specific,
	\vspace{-0.2cm}
	\begin{equation}
	L_{f}h(x)=\frac{\partial h(x)}{\partial x}f(x),\ L_{g}h(x)=\frac{\partial h(x)}{\partial x}g(x).
	\label{lie d}
	\end{equation}

	ZCBF is a special control barrier function that comes with asymptotic stability~\cite{xu2015robustness}. The existence of a ZCBF implies the asymptotic stability and forward invariance of $\mathcal{S}$ as proved in~\cite{xu2015robustness}.

	Based on the ZCBF defined in \noindent\textbf{Definition 2}, a safety barrier for safety-critical systems can be constructed. Concretely, we aim to design a safety barrier for the uncertain system (\ref{eq:system_model}) to keep the state $x$ in the forward invariant safety set $\mathcal{S}$, which requires to hold $\dot{h}(x) \geq -\kappa(h(x))$. 
	
	With the learned disturbances $\widetilde d(x)$ via IGPs and the high confidence interval $\mathcal{D}$ (\ref{high confidence interval}) for the uncertain dynamical system (\ref{eq:system_model}), the following safe control space $K_{rzbf}$ is formulated as shown in our previous work\cite{zheng2020learning}
	\vspace{-0.25\baselineskip}
	\begin{equation}
	K_{rzbf}(x) = \{u\in\mathcal{U}| 
	\mathop{\inf}_{d \in \mathcal{D}(x)}[\dot{h}(x) +\kappa (h(x))]\geq 0\}.
	\label{k_rzbf}
	\vspace{-0.5\baselineskip}
	\end{equation}
	where $h(x)$ is a ZCBF, $\dot{h}(x) = \frac{\partial h(x)}{\partial x}\dot{x}=L_{\hat{f}}h(x) + L_{g}h(x)u + L_{\widetilde{d}}h(x)$, where the $L_{\hat{f}}h(x)$ and $L_{\widetilde{d}}h(x)$ denotes the Lie derivative of $h$ with respect to the known nominal model $\hat{f}$ of $f$ and the learned disturbances $\widetilde d(x)$, respectively. 
	\begin{lemma} 
		Given a set $\mathcal{S}\subset \mathbb{R}^{n}$ defined by (\ref{safety set}) with an associated ZCBF $h(x)$, the control input $u \in K_{rzbf}$ has a probability of at least $(1-\delta)$, $\delta \in (0,1)$, to guarantee the forward invariance of the set $\mathcal{S}$ for the uncertain dynamical system (\ref{eq:system_model})
	\end{lemma}
	
	\begin{proof}
		%The proof proceeds in the following three steps.
		From (\ref{high confidence interval}), there is a probability of at least $(1-\delta)$ such that the bounded model uncertainty $d(x) \in \mathcal{D}(x)$ for all $x\in\mathcal{X}$. Since the control input $u\in\mathcal{U}$ of the safe control space $K_{rzbf}$ satisfies the constraint in (\ref{k_rzbf}), the following result holds with a probability of at least $(1-\delta)$:	\vspace{-0.5\baselineskip}
		\begin{equation}
		\dot h(x) + \kappa(h(x)) \geq 0, \forall x\in\mathcal{S}.
		\vspace{-0.5\baselineskip}
		\end{equation}
		
		As a result, the control input $u\in\mathcal{U}$ of the safe control space $K_{rzbf}$ has a probability of at least $(1-\delta)$ 
		to guarantee the forward invariance of the set $\mathcal{S}$ for the uncertain dynamical system~(\ref{eq:system_model}) as proven in~\cite{ames2017control}.
	\end{proof}
	
	For convenience, with the estimated high confidence interval $\mathcal{D}$ (\ref{high confidence interval}) via GPs, the constraint in (\ref{k_rzbf}) can be equivalently expressed as
	\vspace{-2.5mm}
	\begin{equation}
	L_{\hat{f}}h(x) + L_{g}h(x)u + L_{\mu}h(x) -c_{\delta}|L_{\sigma}h(x)|  \geq  -\kappa(h(x)),
	\label{rcbf constraints}
	\vspace{-1.0mm}
	\end{equation}
	where $L_{\mu}h(x)$ and $L_{\sigma}h(x)$ denote the Lie derivatives of $h(x)$ with respect to $\mu$ and $\sigma$, respectively.
	
	\begin{remark}
		Note that with more informative data collected, the bounded uncertainty $\sigma$ (\ref{var}) will gradually decrease. The proof of such conclusion can be obtained using the partitioned matrix equations, {see Appendix A}. An alternative information theoretic argument is given in \cite{srinivas2012information}\cite{williams2000upper}. Thus, the probability rendering $\mathcal{S}$ forward invariant is much higher than $(1-\delta)$ in most cases.
	\end{remark}
	\subsubsection{Sampling Guidance Scheme}
	In this section, we develop a sampling guidance scheme for CEMPC under uncertain disturbances based on the stability property. Sampling distribution could easily deviate from the low-cost regions in the case of external disturbances. Based on a deviated distribution, the optimization may get stuck in the local minimums. The Lyapunov methods have proven to be an efficient way to improve sampling performance for nonlinear systems in model-based RL \cite{Berkenkamp2016SafeLO,Berkenkamp2017SafeMR}. To guide the sampling distribution of CEMPC to the low-cost region in the case of uncertain disturbances, the stability constraints with regard to the main task, corresponding to the first term in the cost function (\ref{eq:cost_function}), can be constructed based on the CLF $V(x)$:	$\dot{V}(x)\leq -\alpha V(x)$, $\alpha>0$ \cite{khalil2002nonlinear}. 
	
	With learned disturbances $\widetilde d$ estimated by the IGPs, CLF can be utilized to construct a stability control set $K_{rclf}$ \cite{zheng2020learning}
	
	\begin{equation}
	K_{rclf}(x) = \{u\in\mathcal{U}| 
	\mathop{\sup}_{d \in \mathcal{D}(x)}[\dot{V}(x) + \alpha V(x)]\leq 0\},
	\label{prclf}
	%	\vspace{-0.5\baselineskip}
	\end{equation}
	where $\dot{V}(x) = \frac{\partial V(x)}{\partial x}\dot{x}=L_{\hat{f}}V(x) + L_{g}V(x)u + L_{\widetilde{d}}V(x)$ and $\alpha >0$.
	
	With the estimated high confidence interval $\mathcal{D}$ (\ref{high confidence interval}) via GPs, the constraint in (\ref{prclf}) can be simplified as:
	\vspace{-0.15cm}
	\begin{equation}
	L_{\hat{f}}V(x) + L_{g}V(x)u + L_{\mu}V(x)+c_{\delta}|L_{\sigma}V(x)| \leq - \alpha V(x),
	\label{clf constraints}
	\vspace{-0.3\baselineskip}
	\end{equation}
	where $L_{\mu}V(x)$ and $L_{\sigma}V(x)$ denote the Lie derivatives of $V(x)$ with respect to $\mu$ and $\sigma$, respectively.
	%	We found that the gradient-based CLF can provide a good consistent gradient signal to guide the sampling distribution back to the target trajectory at a small computational cost.
	
	\subsubsection{Auxiliary Controller}
	
	{Considering safety constraint (\ref{rcbf constraints}) and stability constraint (\ref{clf constraints}) with slack variables directly in the LB-CEMPC controller, one needs to compute the variance of the predicted states and propagate the variances along the prediction horizon. As a result, it can lead to a heavy computation burden.}
	
	{We design a MI auxiliary controller to integrate constraints (\ref{rcbf constraints}) and (\ref{clf constraints}) to minimally modify the outputs of the MPC and formally satisfy these constraints. The proposed control scheme decouples the requirements of predictive and adaptive control performance in a way of minimal intervention. The safety requirements are considered as soft constraints in the non-differentiable objective function of LB-CEMPC, while they are enforced in a form of CBF in the MI controller. The MI auxiliary controller considers the uncertainty of the disturbance and forms an efficient QP controller to minimally modify the control outputs.}
	
	{With conditions (\ref{rcbf constraints}) and (\ref{clf constraints}),} a QP (\ref{eq:qp_opt1})-(\ref{eq:qp_opt4}) can be constructed to modify the control inputs $u_{MPC}$ of the learning-based CEMPC in a way of minimal intervention. The nonlinear system under environmental uncertainties can be guaranteed safe and guided to a low-cost region of the main task in a high probability by solving the following QP. \vspace{-2mm}
	\begin{alignat}{2}
	\label{eq:qp_opt1}
	u^{*}(x)=&\mathop{\arg\min}_{(u,\varepsilon,\eta) \in {\mathbb{R}^{m+1}}} ||u-u_{MPC}||{^2} + \lambda_{\epsilon}\epsilon^2 +\lambda_{\eta}\eta^2\\
	\mbox{s.t.} \quad
	&A_{cbf}u+b_{cbf} \leq \varepsilon,\label{eq:qp_opt2} \\
	&A_{clf}u+b_{clf} \leq \eta,  \label{eq:qp_opt3} \\
	&u_{min} \leq u \leq u_{max},\label{eq:qp_opt4} 
	\end{alignat}
	where $u_{min}, u_{max}\in\mathcal{U}$ are the lower and upper bound of the control inputs, respectively. $\lambda_{\varepsilon}, \lambda_{\eta}\in\mathbb{R^{+}}$ are penalty coefficients of the slack variables $\varepsilon\in\mathbb{R}$ and $\eta\in\mathbb{R}$, respectively. $A_{cbf} =-L_{g}h(x)$, $b_{cbf}= -L_{f}h(x) - L_{\mu}h(x)+c_{\delta}|L_{\sigma}h(x)| - \kappa  (h(x))$, $A_{clf} =L_{g}V(x)$, $b_{clf}= L_{f}V(x) + L_{\mu}V(x)+c_{\delta}|L_{\sigma}V(x)| - \alpha V(x)$. The feasibility of the QP (\ref{eq:qp_opt1})-(\ref{eq:qp_opt4}) can be ensured with the slack variables, while the violations of safety constraints can be heavily penalized as long as the corresponding coefficients are large enough. 
	
	The designed MI auxiliary control scheme is shown in \textbf{Algorithm \ref{alg:Framwork_2}}. We can always directly apply the control outputs $u_{MPC}$ of the learning-based CEMPC (\textbf{Algorithm \ref{alg:Framwork_1}}, line 10) if the $u_{MPC}$ satisfies the constraints (\ref{rcbf constraints}) and (\ref{clf constraints}). Otherwise, it is modified by solving the QP (\ref{eq:qp_opt1})-(\ref{eq:qp_opt4}).
	%%%%%%%%%%%%%%%%%%%%%%%%%%%%%%%%%%%%%%%%%%%%%%%%%%%%%%%%%%%%%%%%%%%%%%%%%%%%%%%%
	\begin{algorithm}[htb] 
		\caption{MI Auxiliary Control Scheme} 
		\label{alg:Framwork_2} 
		\begin{algorithmic}[1] 
			\REQUIRE ~~\\ 
			$x_k$: Current state, \\
			$u_{MPC}$: Control inputs from learning-based CEMPC.			
			\STATE $u_k$ $\leftarrow$ $u_{MPC}$
			\IF{(\ref{rcbf constraints}) and (\ref{clf constraints}) Infeasible } 
			\STATE  $u_{k}$ $\leftarrow$ Solve QP (\ref{eq:qp_opt1})-(\ref{eq:qp_opt4}) with $x_k$ and $u_{MPC}$
			\ENDIF 
			%		\ENDWHILE
			\label{code:recentStart}
			\label{code:fram:select}
		\end{algorithmic}
	\end{algorithm}
	
	%	\begin{remark} Note that the solution $u^*$ to the QP in (\ref{eq:qp_opt1})-(\ref{eq:qp_opt4}) is always feasible because the slack variables $\epsilon$ and $\eta$ can ensure no conflict among the safety (\ref{eq:qp_opt2}), stability (\ref{eq:qp_opt3}) and control input constraints (\ref{eq:qp_opt4}).
	%	\end{remark}
	
	\begin{remark}
		Note that the optimization~(\ref{eq:qp_opt1}) is not sensitive to the parameters $\lambda_{\varepsilon}$ and $\lambda_{\eta}$.
		The violation of the safety (\ref{rcbf constraints}) and stability constraints (\ref{clf constraints}) can be heavily penalized as long as the $\lambda_{\varepsilon}$ and $\lambda_{\eta}$ are large enough (e.g. $\lambda_{\varepsilon}=10^{30}$,$\lambda_{\eta}=10^{20}$). Besides, $\lambda_{\varepsilon}$ can be set extremely larger than $\lambda_{\eta}$ to make the safety constraints much stricter.
	\end{remark}
	
	\section{Simulation Studies}
	\vspace{-.04cm}
	\label{section:experiment}
	In this section, the proposed control architecture is verified on a task of simultaneous trajectory tracking and obstacle avoidance of a quadrotor. For a safety-critical quadrotor with a limited sensing range, avoiding an uncertain number of detected obstacles around or on the trajectory can be conveniently encoded as some {non-differentiable} running-cost terms into the objective function. It is difficult to describe such a complex task and trade off the tracking and safety well with a simplified differentiable objective function. Besides, a quadrotor is prone to uncertain wind disturbances, which are hard to be accurately modeled. Uncertain wind disturbances not only pose a critical challenge to achieve accurate tracking, but also may cause the quadrotor to collide in a cluttered obstacle field. 
	
	\subsection{Quadrotor Dynamics and Control}
	The quadrotor is a well-modeled dynamical system with torques and forces generated by four rotors and gravity. The Euler angles (roll $\phi$, pitch $\theta$, and yaw $\psi$) are defined with the ZYX convention. The attitude rotation matrix $R\in SO(3)$ from the body frame $\mathcal{B}$ to the global frame $\mathcal{W}$ can be written as~\cite{hehn2015real}:\
	\begin{equation}
	R=\left[\begin{matrix}c\theta c\psi&s\phi s\theta c\psi-c\phi s\psi&c\phi s\theta c\psi+s\phi s\psi\\c\theta s\psi&s\phi s\theta s\psi+c\phi c\psi&c\phi s\theta s\psi-s\phi c\psi\\-s\theta&s\phi c\theta&c\phi c\theta\\\end{matrix}\right],
	\end{equation}
	where $s$ and $c$ denote $sin$ and $cos$, respectively.
	
	%	Approximating the drag as a linear function of the airspeed is a reasonable simplification at low speed as proposed in~\cite{svacha2017improving}.  
	
	The nonlinear quadrotor system can be modeled as following \cite{shi2019neural}:\vspace{-0.2cm}
	\begin{align}
	&\dot{p} =v, \label{Za}\\
	& m{\dot{v}} = mge_3 + Rf_u+d_w, \label{Zb}\\
	&\dot{R} = RS(\omega)\label{Zc},\vspace{-0.3cm}
	\end{align}
	where $m$ and $g$ denote the mass and the gravity acceleration, respectively. $e_3=[0,0,1]^T$ is the unit vector, $p=[p_x,p_y,p_z]^{T}$ and $v=[v_x,v_y,v_z]^{T}$ denote translational position and velocity in $\mathcal{W}$, respectively. $f_u=[0,0,f_T]^T$ with $f_T$ the total thrust generated from the four rotors in $\mathcal{B}$, and {$S({}\cdot)$} is skew-symmetric mapping. The uncertain wind disturbances acting on the quadrotor dynamics is represented as $d_w =K_{drag}(v_w-v)$, where $v_w \in \mathbb{R}^3$ is the velocity of wind disturbances in $\mathcal{W}$ and $K_{drag} \in \mathbb{R}^{3 \times 3}$ is the drag coefficient diagonal matrix. We define in the dynamics equation (\ref{eq:system_model}) the state $x=[p_x, p_y, p_z, v_x, v_y, v_z, \phi, \theta, \psi]^{T}$ and the control input $u = [f_T, \omega^{T}]^T$, where $\omega = [\omega_{x}, \omega_{y}, \omega_{z}]^{T}$ is the the body rotational rates. It is assumed that the body rotational rates are directly controllable through the fast response onboard controller of commercial quadrotors. 
	
	For the task of simultaneous trajectory tracking and obstacle avoidance, the cost function $\mathcal{L}$ in the objective function (\ref{eq:score_function}) can be defined as   \vspace{-0.1cm}
	\begin{equation}
	\mathcal{L}(x) = \iota(x-x_d)^TQ\iota(x-x_d) + \sum_{i}^{N_0}w_{i}\frac{\mathbb{I}_{C_i}}{r_i}% + \sum_{i}^{5}w_{oi}\mathbb{I}_{O_i},
	\label{cost_func}\vspace{-0.3cm}
	\end{equation}
	\begin{equation}
	C_i = \{x|\ ||r_i||<0.8\}, 
	\end{equation}
	%	\begin{equation}
	%	O_1 =\{x|\ |\phi|>0.9\},\ O_2 =\{x|\  |\theta|>0.9 \}\\
	%	\label{angel limits} 
	%	\end{equation}
	%	\begin{equation}
	%	O_3 =\{x|\ |v_x|< 3 \},\ O_4 =\{x|\ |v_y|< 3 \},\ O_5 =\{x|\ |v_z|< 3 \}\\
	%	\label{vel limits} 
	%	\end{equation}
	where $\iota={\operatorname{diag}(1,1,1,1,1,1,0,0,0)}$ extracts the position and velocity states from the state $x$, the weight coefficient matrix $Q={\operatorname{diag}(8.5,8.5,8.5,1.5,1.5,1.5)}$, $N_0$ denotes the number of the detected obstacles, and the weight coefficient $w_{i}$ is a positive constant, $\forall \ i=1,...,N_0$.  $x_d=[p_d^{T},{\dot{p}_d}^{T},\phi_d,\theta_d,\psi_d]^{T}$ is the desired state, where $\phi_d,\theta_d,\psi_d \in \mathbb{R}$ are the desired attitudes, $r_i$ is the shortest Euler distance from the quadrotor to the $i$-th detected obstacle, and $\mathbb{I}_{C_i}$ is an indicator function that will be turned on if the state {is} in the set $C_i$. 
	
	\begin{remark}
		Note that the second term in (\ref{cost_func}) is designed to show the predictivity inherent in the CEMPC for obstacle avoidance. The trade-off between the safety and tracking performance can be adjusted by the weight $w_{i}$ in (\ref{cost_func}).
	\end{remark}
	\begin{remark}
		Note that designing a differentiable cost function for this task would be nontrivial! Besides, composing a cost function with different nonlinear cost terms may lead to local minima, which brings difficulties to the optimization. In this case, specifying a {non-differentiable} cost function with several interpretable terms and different weights can be convenient according to the importance of task requirements.
		%		 Many of the task requirements have gradient-free components (e.g. the norm operators).

	\end{remark}
	\subsection{Simulation setup}
	
\begin{table}[]
\centering
\caption{Parameter Table For Simulation Setup}
\label{tab:param}
\begin{tabular}{l|l}
\hline
{Parameter}                                                  & {Value}                           \\ \hline
{The mass of quadrotor $m$}                                  & {$0.08kg$}                        \\
{The arm length from the center of mass to each motor}       & {$0.11m$}                         \\
{The maximum total thrust $f_{T_{max}}$}                     & {$1.3N$}                          \\
{The maximum body rotational rates $\omega_{max}$}           & {$[3.49,3.49,5.24]^\intercal$}    \\
{The minimum body rotational rates $\omega_{min}$}           & {$[-3.49,-3.49,-5.24]^\intercal$} \\
{The detection range for the obstacle}                       & {$2m$}                            \\
{The simulation time}                                        & {$20s$}                           \\
{The control frequency}                                      & {$50HZ$}                          \\
{The hyperparameters $L$ of GP kernel}                       & {$1$}                             \\
{The hyperparameters $\sigma_f$ of GP kernel}                & {$1$}                             \\
{The max size of IGP dataset $n$}                            & {$20$}                            \\
{The confidence parameter $c_{\delta}$}                      & {$3$}                             \\
{The predictive time steps $H$ of CEMPC per iteration}      & {$20$}                            \\
{The number of the samples $M$ for CEMPC per iteration}      & {$100$}                           \\
{The size of the elite set $K$ for CEMPC per iteration}      & {$10$}                            \\
{The number of iterations $N$ of CEMPC}                      & {$5$}                             \\
{The minimum variance bound for optimization $\Sigma_{min}$} & {$0.001$}                         \\
{The update rate $\beta$}                                    & {$0.25$}                          \\
{The slack variable $\lambda_{\varepsilon}$ in the QP}       & {$10^{30}$}                       \\
{The slack variable $\lambda_{\eta}$ in the QP}              & {$10^{20}$}                       \\ \hline
\end{tabular}
\end{table}
	
	A simulation platform is created using Python 3.6 on an Intel Xeon X5675 CPU with 3.07 GHz clock frequency. A quadrotor model of Blade mQX quadrotor is used with the parameters set referred to \cite{cabecinhas2014globally}.
	{The main simulation parameter values are shown in Table \ref{tab:param}.}
	%The mass is $m=0.08kg$. The arm length from the center of mass to each motor is $0.11 m$. The maximum total thrust $f_{Tmax}=1.3N$. The maximum and minimum body rotational rates are $\omega_{max}=[3.49,3.49,5.24]^Trad/s$ and $\omega_{min}=[-3.49,-3.49,-5.24]^Trad/s$, respectively. The detection range for the obstacle is set to be omnidirectional $2m$. The simulation time is set to be $20 s$ with a control frequency of $50Hz$. 
	
	A wind model in \cite{cole2018reactive} is utilized to evaluate the algorithm performance. Wind velocity consists of a constant component $v_c$ and a turbulent component $v_t$, i.e.,  $v_w=v_c+v_t$. The turbulence wind uses the von Kármán velocity model defined analytically in the specification MIL-F-8785C \cite{moorhouse1980us}, with the specific low-altitude model for the model parameters. The drag coefficient diagonal matrix $K_{drag}={\operatorname{diag}(0.03,0.03,0.03)^T}$. Four magnitudes of constant wind components are used to validate the trajectory tracking performance, as shown in Table \ref{table:table_results}.

	Three GPs are built to estimate the effects of the unknown wind disturbances $d_w$. Each GP uses the same squared-exponential kernel.% with parameters $L=1$ and $\sigma_f=1$. The size of the IGP dataset is fixed at $n=20$. We set $c_{\delta}=3$ in (\ref{high confidence interval}) to generate high confidence intervals of the estimation. 

	The quadrotor with a limited sensing range is required to track a reference trajectory while avoiding obstacles under varying wind disturbances in \textbf{{three} scenarios}. The initial state of the quadrotor is set as the initial position of the reference trajectory with zero velocity and attitude. 
	
	\noindent\textbf{Scenario 1}\noindent~is designed to validate the effectiveness of the proposed learning-based CEMPC method with the MI controller. The reference trajectory is given as a spiral curve $p_d(t)=[2sin(0.5t), 2-2cos(0.5t),0.2t]^{T}$ and $\psi_d(t) = 0$. It lies in a dense cluttered obstacle field under time-varying wind disturbances, where a moving obstacle flies along the reference trajectory to the quadrotor with a speed of $0.8 m/s$. The weight coefficients are set $w_{i} =10 , \forall \ i=1,...,N_0$ in (\ref{cost_func}).
	
	\noindent\textbf{Scenario 2}\noindent~is studied to further validate the safety and tracking performance trade-off in the design of the {non-differentiable} cost function (\ref{cost_func}), where the weight coefficient $w_{i}$ is set to different orders of magnitude, $\forall \ i=1,...,N_0$. The difference to \textbf{Scenario 1} is that there are only one static obstacle and one dynamic obstacle along the reference trajectory. It is devised to clearly show the proactive ability inherent in the CEMPC framework for obstacle avoidance. 
	%	The difference to Scenario 1 is that there are one static and one dynamic obstacle there is a static obstacle crossing the reference trajectory. wind, obstacles 
	
	\noindent{\textbf{Scenario 3}\noindent~is studied to illustrate the efficiency of LB-CEMPC with different prediction horizons and sample sizes. The reference trajectory is an unsmooth trajectory made of three straight lines $l_1=30~\mathrm{m}$ connected by two perpendicular straight lines $l_2=30~\mathrm{m}$. There is no obstacle in \textbf{Scenario 3} and the reference velocity is set $8m/s$. A trade-off between efficiency and tracking performance can be obtained from the simulation results of \textbf{Scenario 3}.}
	
	%The number of the samples and the size of the elite set for CEMPC per iteration are set $M=100$ and $K=10$, respectively. The number of iterations of CEMPC is $N=5$. The minimum variance bound for optimization is set as $\Sigma_{min}=0.001$ and the update rate $\beta=0.25$.
	
	The barrier function $h$ can be constructed with the distance from the quadrotor to the obstacles within its sensing region. This distance can be obtained with the largest ellipsoidal region of obstacle-free space, which can be efficiently computed using the IRIS algorithm~\cite{deits2015computing} through semi-definite programming. Specifically, an ellipsoid can be represented as an image of the unit ball:
	\begin{equation}
	E(C, \zeta) =\{Co+\zeta\ |\ \Vert o \Vert=1\},
	\end{equation}
	where $o\in\mathbb{R}^{3\times 1}$, $o^{T}o=1$ denotes a unit ball, the mapping matrix $C\in\mathbb{R}^{3\times 3}$ and the offset vector $\zeta \in\mathbb{R}^{3\times 1}$ can be obtained using the IRIS algorithm. For the vector $\epsilon$ on the ellipsoid $\epsilon \in E$, we have $(\epsilon - \zeta)^{T}{C^{-1}}^{T} C^{-1}(\epsilon - \zeta)=1$. The CBF can be constructed to enforce the quadrotor to stay within the safety ellipsoid region as 
	\begin{equation}
	h(x) =1-(\iota_p x- \zeta)^{T}{C^{-1}}^{T} C^{-1}(\iota_p x - \zeta).
	\label{cbf}
	\end{equation}
	where $\iota_p={\operatorname{diag}(1,1,1,0,0,0,0,0,0)}$ extracts the position from the state $x$.
	
	For the main task of trajectory tracking, the CLF is designed as 
	\begin{equation}
	V(x)=(\iota_p x-p_d)^{T}Q_p(\iota_p x-p_d)+(\iota_v x-{\dot{p}_d})^{T}Q_v(\iota_v x-{\dot{p}_d}),
	\end{equation}
	where $Q_p={\operatorname{diag}(0.8,0.8,0.8)}$, $Q_v={\operatorname{diag}(0.2,0.2,0.2)}$ and  $\iota_v={\operatorname{diag}(0,0,0,1,1,1,0,0,0)}$ extracts the velocity from the state $x$. The extended $\mathcal {K}$ class function $\kappa$ is chosen as $\kappa(h(x)) = 10h(x)$, and the positive constant $\alpha = 0.1$.
	%The penalty coefficients of the slack variables in the QP (\ref{eq:qp_opt1}) are set $\lambda_{\varepsilon}={10^{30}}$ and $\lambda_{\eta}={10^{20}}$.
	The QP is solved with CVXOPT solver~\cite{andersen2013cvxopt}. 
	
	\subsection{Results}
	\label{section:results}
	To validate the effectiveness of the proposed control scheme, an ablation study is conducted to assess that the proposed method is able to: 1) learn and adapt to the uncertain environmental disturbances, 2) handle the task with a {non-differentiable} objective function, 3) achieve safe control with a low tracking error, and 4) provide a way to trade off between safety and tracking performance. We compare the following four methods:
	
	\begin{itemize}
		\item \textbf{CEMPC}: A CEM-based MPC without the IGPs for learning the uncertain disturbances.
		\item \textbf{LB-CEMPC}: A learning-based CEMPC without the MI auxiliary controller.
		\item \textbf{LB-CEMPC-CBF}: The proposed learning-based CEMPC without the sampling guidance scheme {(\ref{clf constraints})} in the MI auxiliary controller.
		\item \textbf{LB-CEMPC-MI}: The proposed learning-based CEMPC with the designed MI auxiliary controller.
	\end{itemize}

	\begin{figure*}[t]
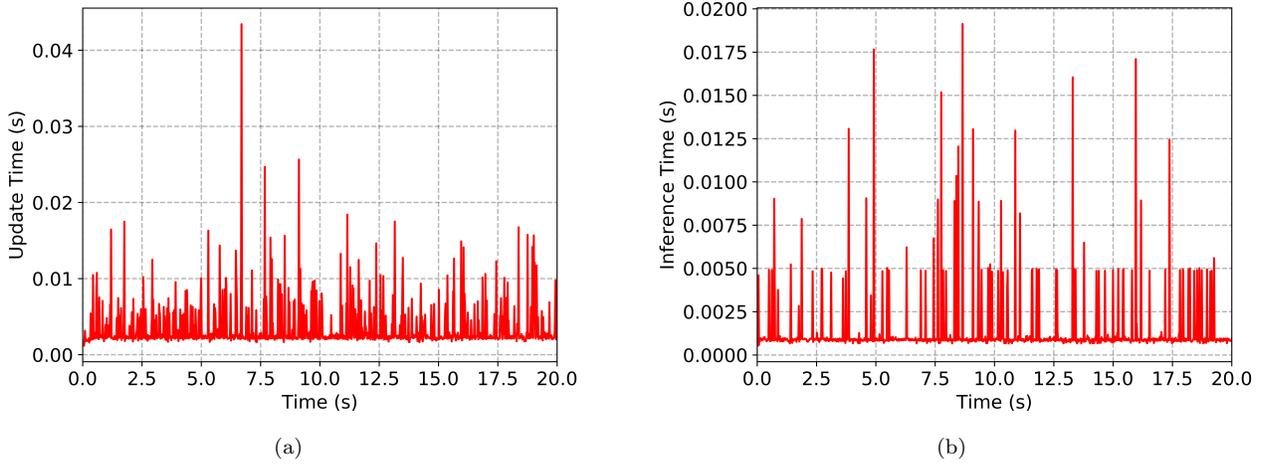

		\centering
		\subfigure[]{
			\label{fig:GP update}
			\includegraphics[scale=0.045]{gp_learn_time-eps-converted-to.pdf}
		}\hspace{5mm}
		\subfigure[]{
			\label{fig:GP inf}
			\includegraphics[scale=0.045]{gp_inf_time-eps-converted-to.pdf}
		}
		\caption{(a) The time of learning, and (b) the time of inference of the IGP under the Wind-4. }
		\label{fig:GP_wind-4}
	\end{figure*}

	\begin{figure}[t]
		\centering%[width=11cm]
		\includegraphics[scale=0.15]{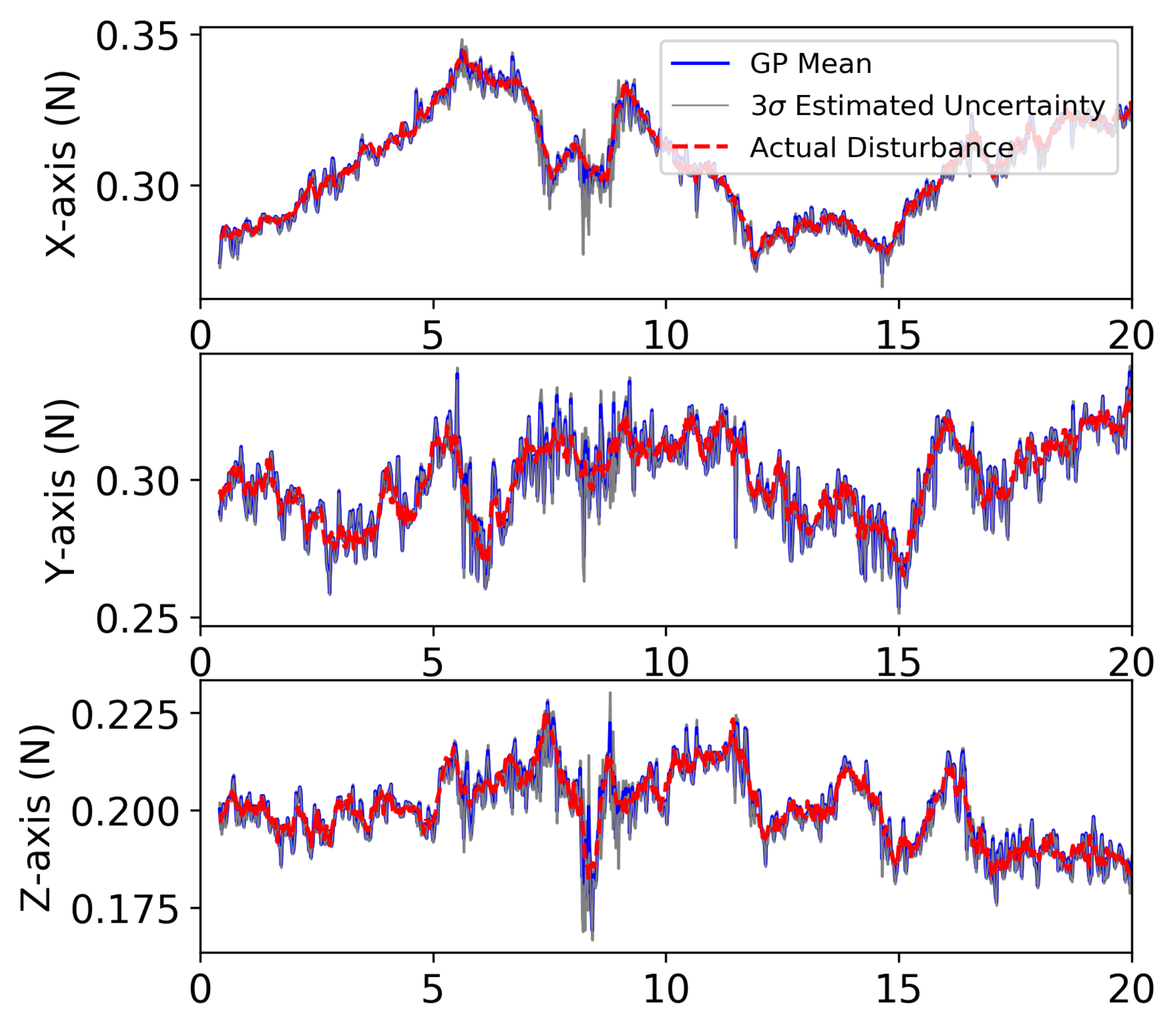}	\vspace{-0.2cm}
		\caption{{The wind disturbances estimated by IGPs in three axes on the spiral trajectory under the Wind-4.}}
		\label{fig:GP}
	\end{figure}
	
	\subsubsection{Learning Performance}
	Figures \ref{fig:GP update} and \ref{fig:GP inf} {show} the time of learning and inference with the IGP at each iteration under Wind-4. The learning time is less than $0.02s$ most of the time and the inference time keeps below $0.02s$. It shows that the IGP can be used as an online learning technique with fast learning and inference time. Note that the code in Python has not been optimized for speed and can be accelerated in a C++ implementation.

	The uncertain wind disturbances in three axes modeled by GPs are shown in Fig.~\ref{fig:GP}. It can be seen that the IGPs can estimate well the actual wind disturbances with turbulence. The actual disturbances lie within the uncertainty bounds of the estimations. 
	
	\subsubsection{Trajectory Tracking Performance}
	
	\begin{table}[htbp]
		\renewcommand{\arraystretch}{1.1}
		\scriptsize
		\caption{
			Statistics of RMS Errors (in meter) with Different Control Schemes and Configurations. }
		\label{table:table_results}
		\centering
		\begin{tabular}{ccc|cccc}
			\hline
			Prediction Horizon $T_h$ & High-level &  Scheme& 
			\begin{tabular}{@{}c@{}} Wind-1\\$v_c=5m/s$ \end{tabular}  & 
			\begin{tabular}{@{}c@{}}  Wind-2\\$v_c=8m/s$\end{tabular}  
			& 
			\begin{tabular}{@{}c@{}}  Wind-3\\$v_c=10m/s$ \end{tabular}  & 
			\begin{tabular}{@{}c@{}}  Wind-4\\$v_c=12m/s$\end{tabular}  \\
			\hline
			0.2 s            &	CEMPC            & ---           & 0.704 & 0.818 & 1.164  & 1.872 \\
			0.2 s          &	LBCEMPC            & ---          & 0.630 & 0.631 & 0.636  & 0.644\\
			0.2 s          &	LBCEMPC            &  CBF      & 0.632 & 0.638 & 0.640  & 0.636\\
			0.2 s           &	\textbf{LB-CEMPC}   & \textbf{MI}   & \textbf{0.354} & \textbf{0.195} & \textbf{0.246} & \textbf{0.388}  \\
			
			0.4 s           &	CEMPC            & ---           & 0.662 & 0.765 & 1.257  & 1.950 \\
			0.4 s            &	LBCEMPC            & ---         & 0.593 & 0.608& 0.601  & 0.636\\
			0.4 s          &	LBCEMPC            &  CBF      & 0.578 & 0.625 & 0.601  & 0.618\\
			0.4 s           &	\textbf{LB-CEMPC}   & \textbf{MI}   & \textbf{0.215} & \textbf{0.124} & \textbf{0.214} & \textbf{0.349} \\
			0.6 s           &	CEMPC            & ---           & 0.631 & 0.737 & 1.260  & 2.077\\
			0.6 s           &	LBCEMPC            & ---          & 0.582 & 0.566 & 0.596 & 0.579 \\
			0.6 s          &	LBCEMPC             & CBF & 0.556 & 0.583  & 0.604 & 0.625\\
			0.6 s          &	\textbf{LB-CEMPC}   & \textbf{MI}   & \textbf{0.170} & \textbf{0.178} & \textbf{0.271} & \textbf{0.341}  \\
			\hline
		\end{tabular}
	\end{table}

	\begin{figure}[t]
		\centering
		\includegraphics[scale=0.15]{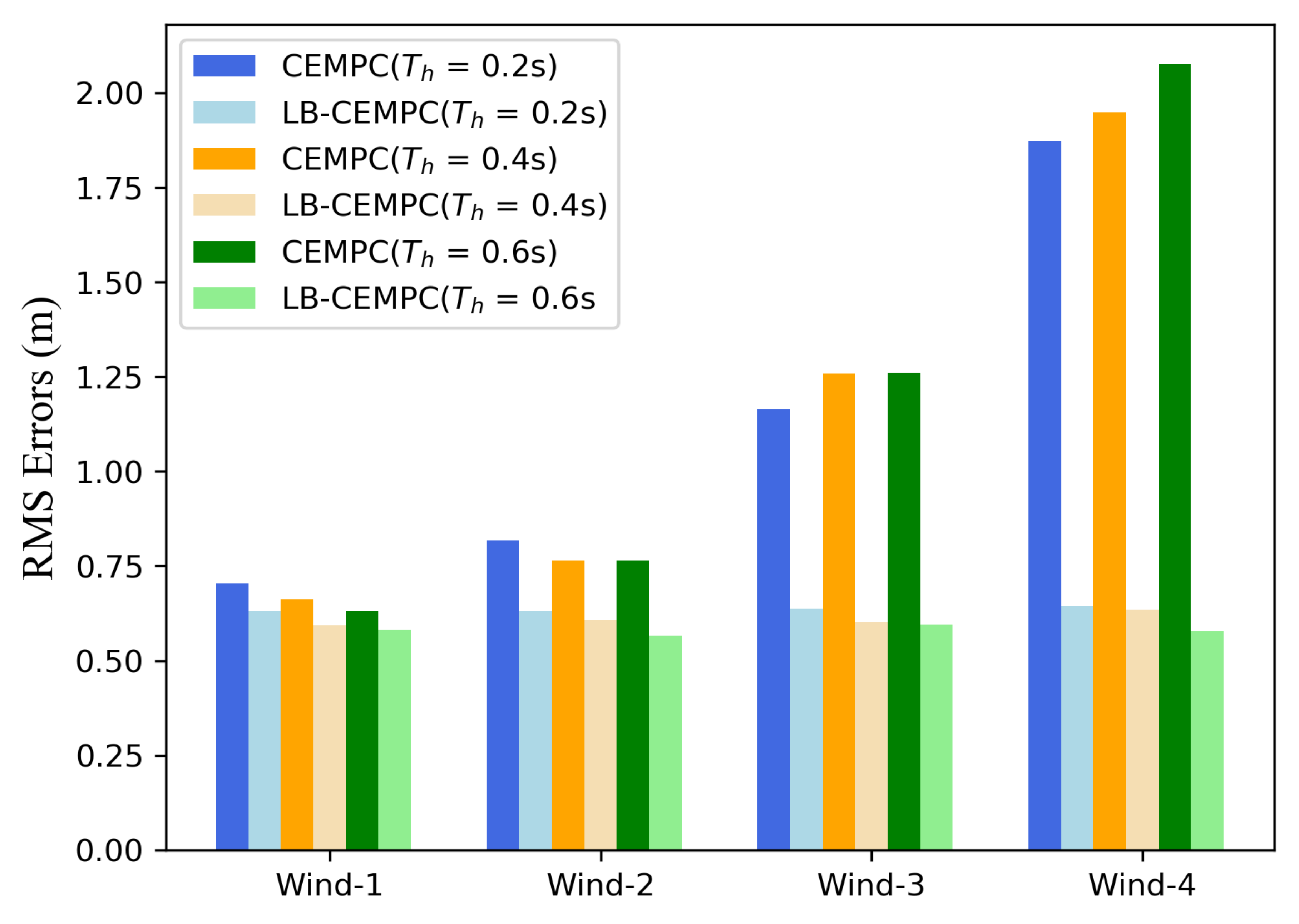}	\vspace{-2mm}
		\caption{The RMS errors of the quadrotor trajectory tracking using the LB-CEMPC and the CEMPC controllers with different prediction horizons under wind disturbances. The proposed LB-CEMPC outperforms the baseline CEMPC in four settings of time-varying wind disturbances.}
		\label{fig:mpcrmse}\vspace{3mm}
	\end{figure}

	\begin{figure*}[!ht]
		\centering
		\subfigure[Wind-1.]{
			\label{fig:MPC_error_wind1}
			\includegraphics[scale=0.045]{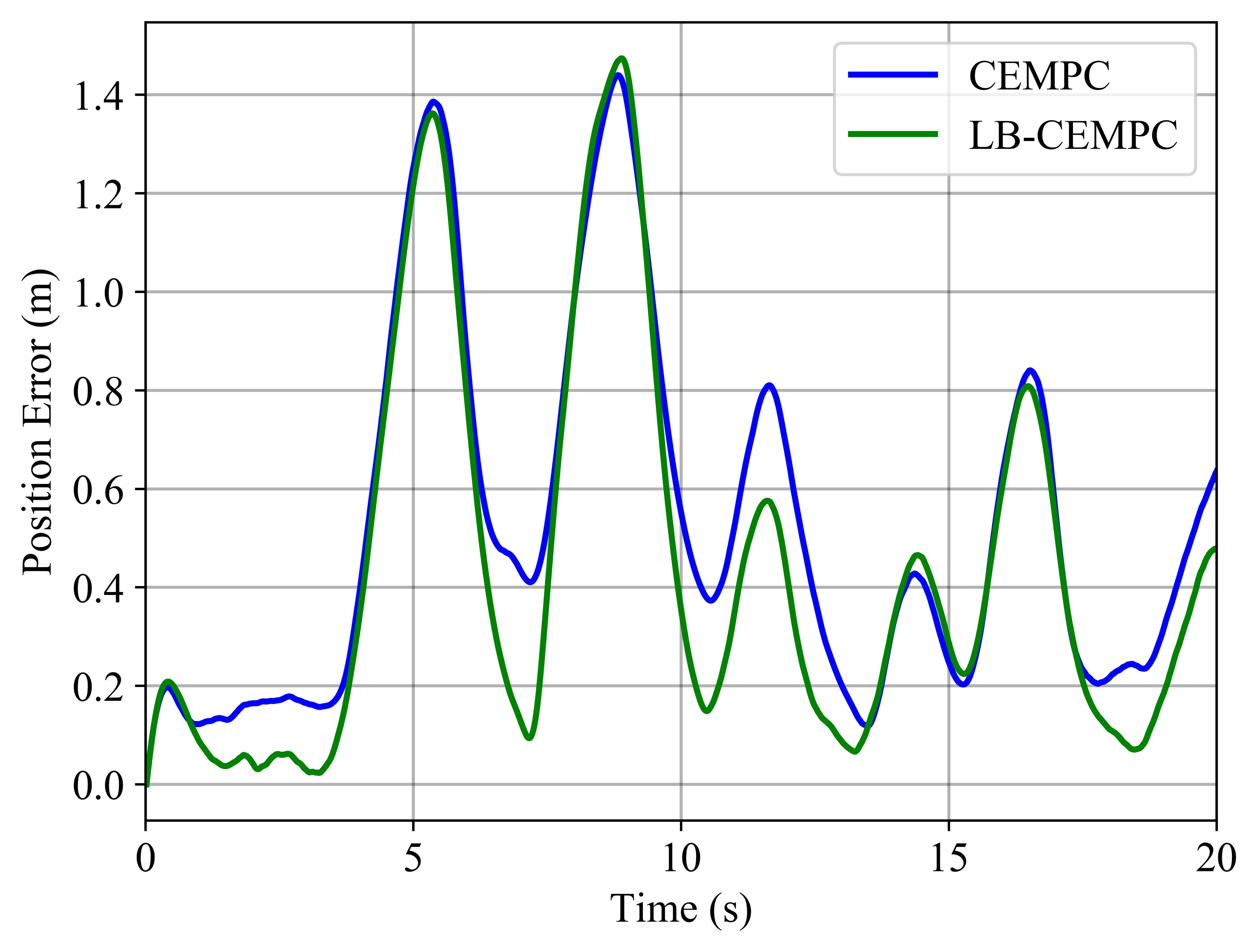}
		}\hspace{5mm}
		\subfigure[Wind-4.]{
			\label{fig:MPC_error_wind4}
			\includegraphics[scale=0.045]{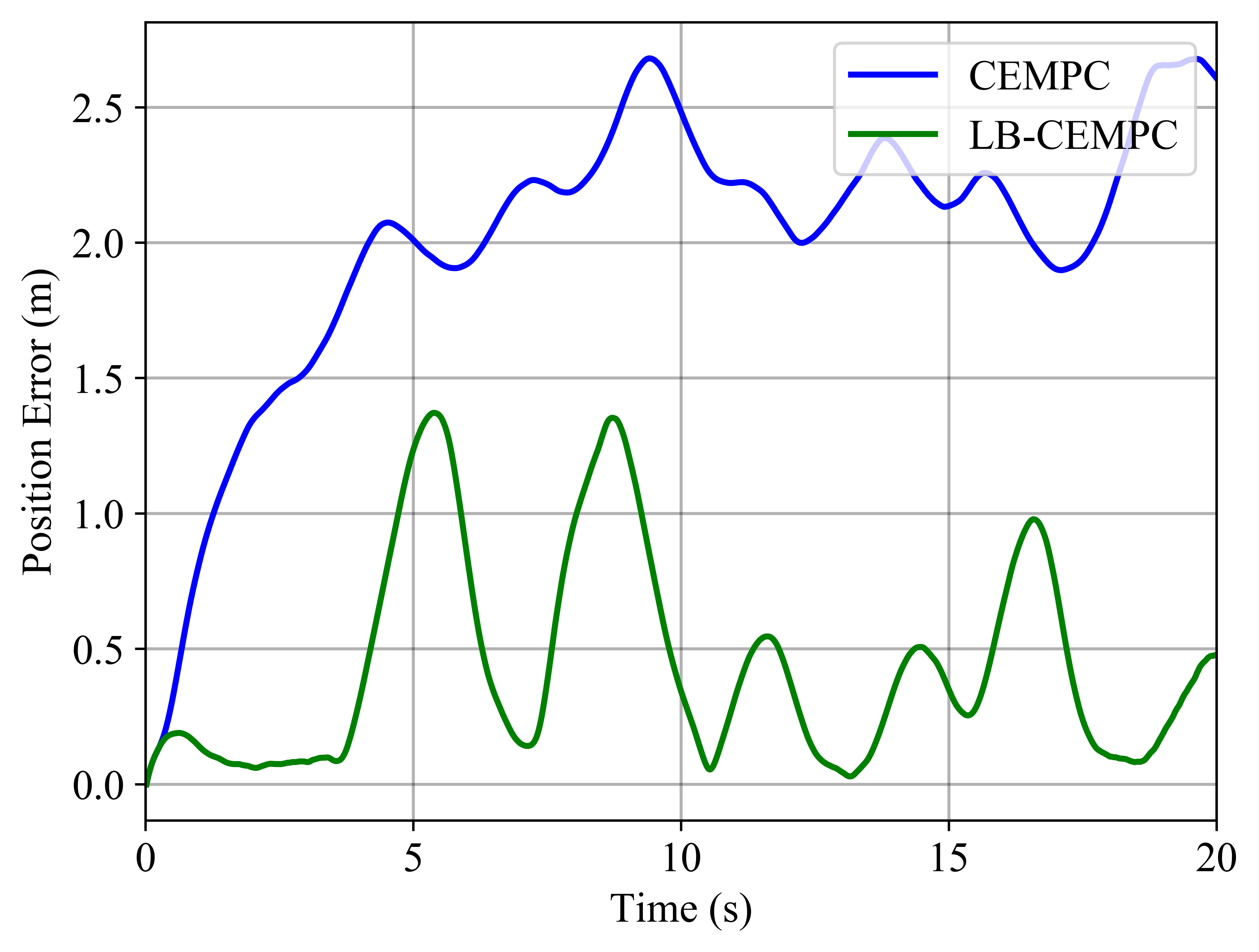}
		}\vspace{-5mm}
		\caption{Evolution of trajectory tracking errors using the CEMPC and the LB-CEMPC under the mild Wind-1 and the strong Wind-4 disturbances.}
		\label{fig:MPC_error}
	\end{figure*}
	
	As shown in Table~\ref{table:table_results} and Fig.~\ref{fig:mpcrmse}, the CEMPC with a longer horizon achieves smaller tracking RMS error under mild Wind-1 or Wind-2, due to the robustness from the predictivity and receding horizon optimization inherent in the MPC \cite{mayne2014model}. However,  under larger wind disturbances, e.g. Wind-3 and Wind-4, the RMS error of the CEMPC instead increases as the prediction horizon gets longer, since the accumulation of the model error along the multi-step predictions degrades the control performance. The tracking errors of the CEMPC and LB-CEMPC with a prediction horizon of $T_h=0.6s$ under Wind-1 and Wind-4 are shown in Fig.~\ref{fig:MPC_error}. It can be seen that the tracking errors of the CEMPC are similar to that of the LB-CEMPC under Wind-1, while the tracking errors of the CEMPC get quite high without the IGPs under the stronger Wind-4. These results indicate that the LB-CEMPC benefits from the IGPs learning and compensating the wind disturbances.
	
	\begin{figure*}[t]
		\centering
		\subfigure[Tracking error.]{
			\label{fig:Tracking error}
			\includegraphics[scale=0.045]{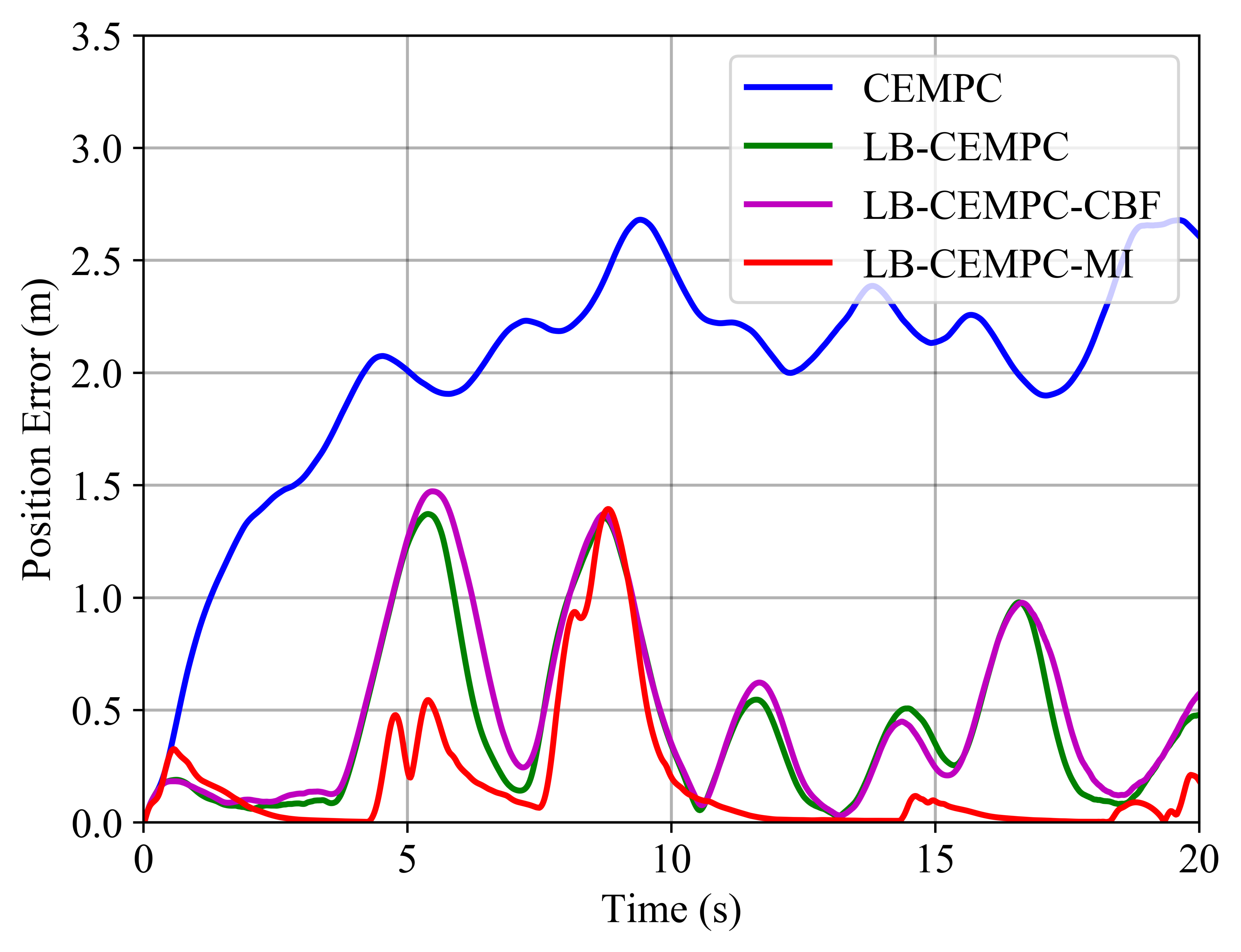}
		}\hspace{5mm}
		\subfigure[Velocity]{
			\label{fig:velocity}
			\includegraphics[scale=0.045]{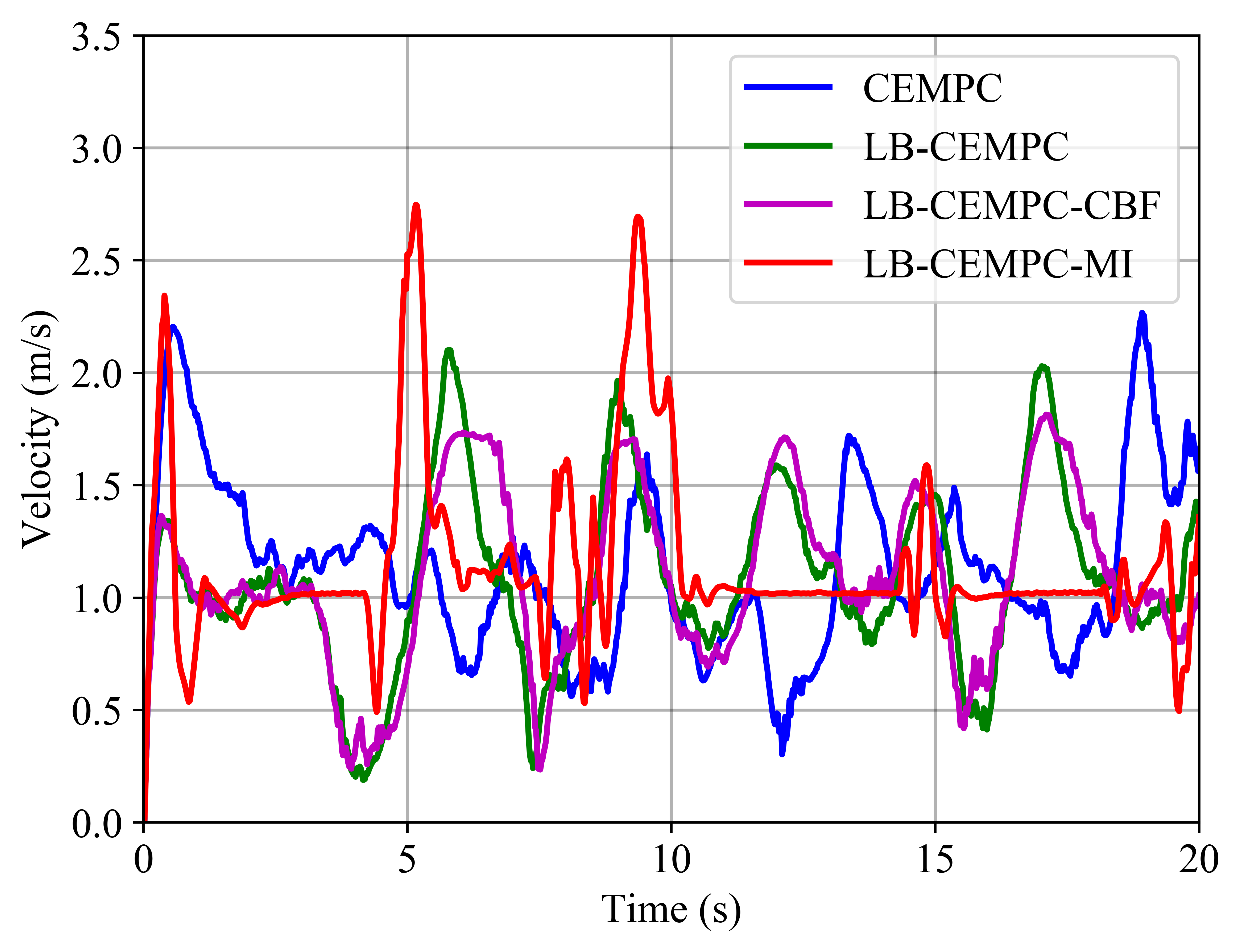}}
		\caption{Safe trajectory tracking in \textbf{Scenario 2}. (a) Tracking error with different controllers, and (b) Tracking velocity with different controllers under Wind-4 disturbances.}
		\label{fig:Scenario1_performance}
	\end{figure*}
	\begin{figure*}[!ht]
		\centering
		\subfigure[]{
			\label{fig:UAV Position}
			\includegraphics[scale=0.13]{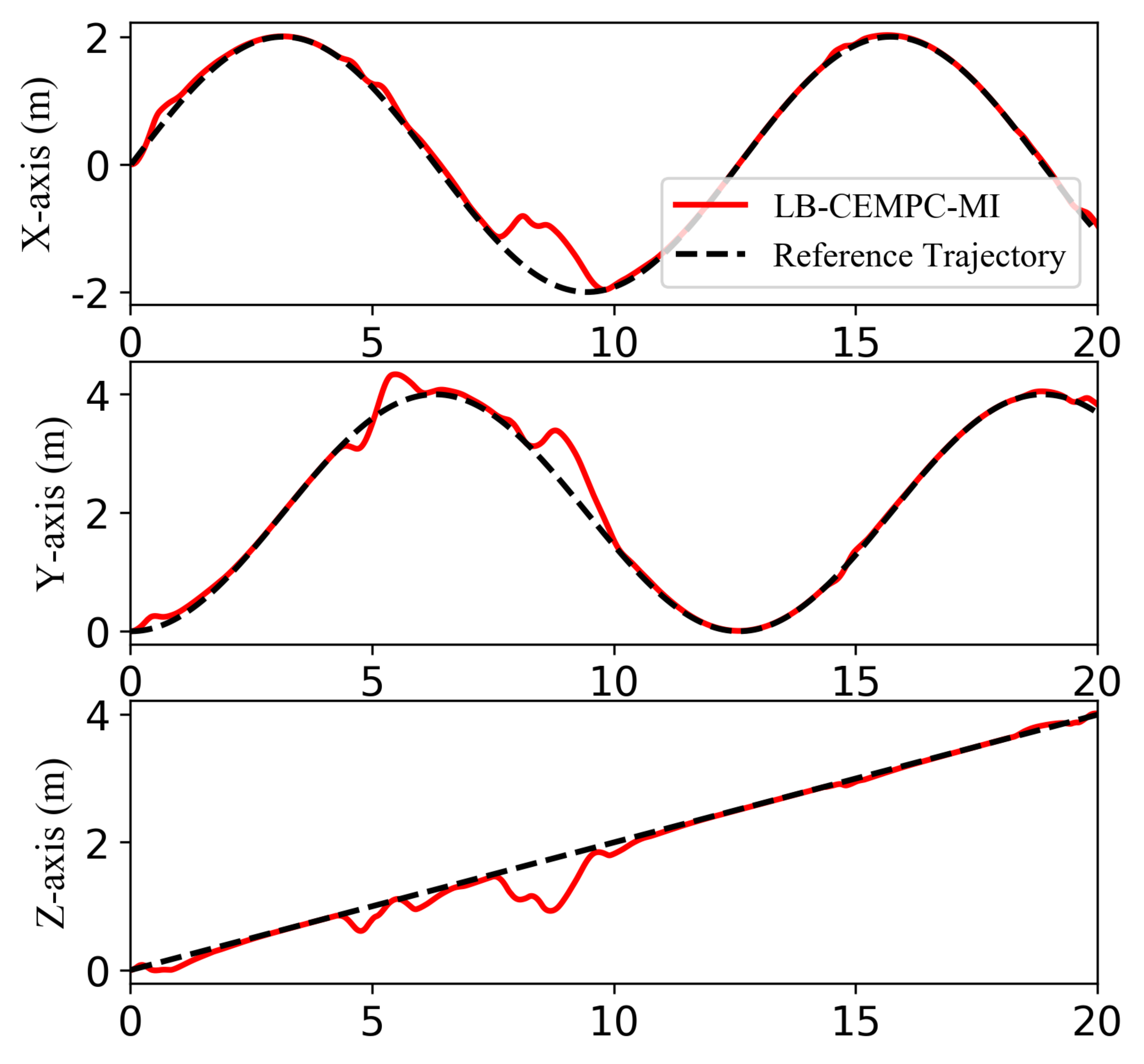}
		}\hspace{5mm}
		\subfigure[]{
			\label{fig:cbf1}
			\includegraphics[scale=0.045]{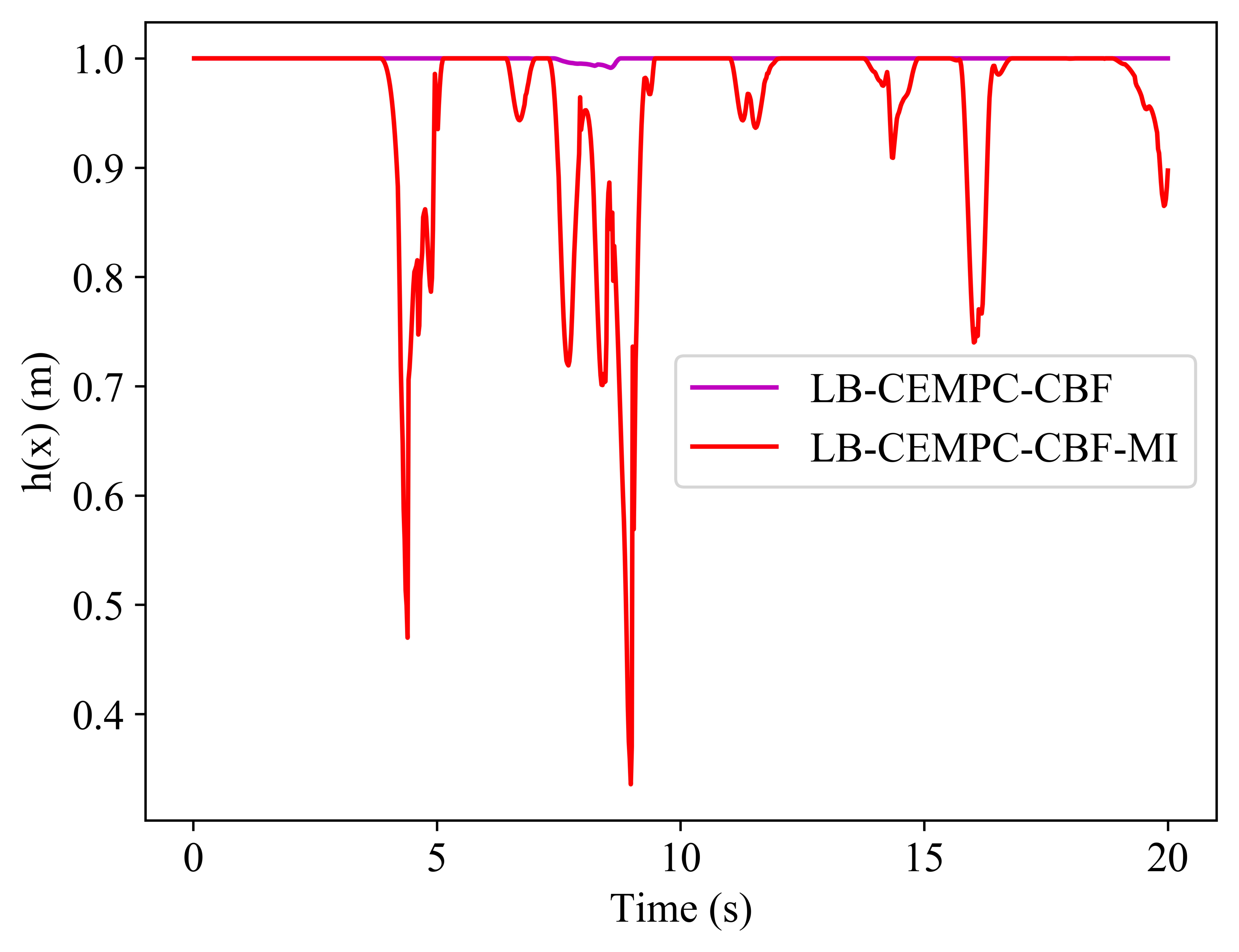}}
		\caption{Safe trajectory tracking in \textbf{Scenario 2}. (a) Position of the quadrotor with the proposed LB-CEM-MI control scheme, and (b) values of CBF in the LB-CEMPC-CBF and the LB-CEMPC-MI under Wind-4 disturbances.}
		\label{fig:Scenario1_MI}
	\end{figure*}

	\begin{figure*}[!ht]
		\centering
		\hspace{-5mm}	\subfigure[Snapshot at t = 6s.]{
			\label{fig:3D_1}
			\includegraphics[scale=0.09]{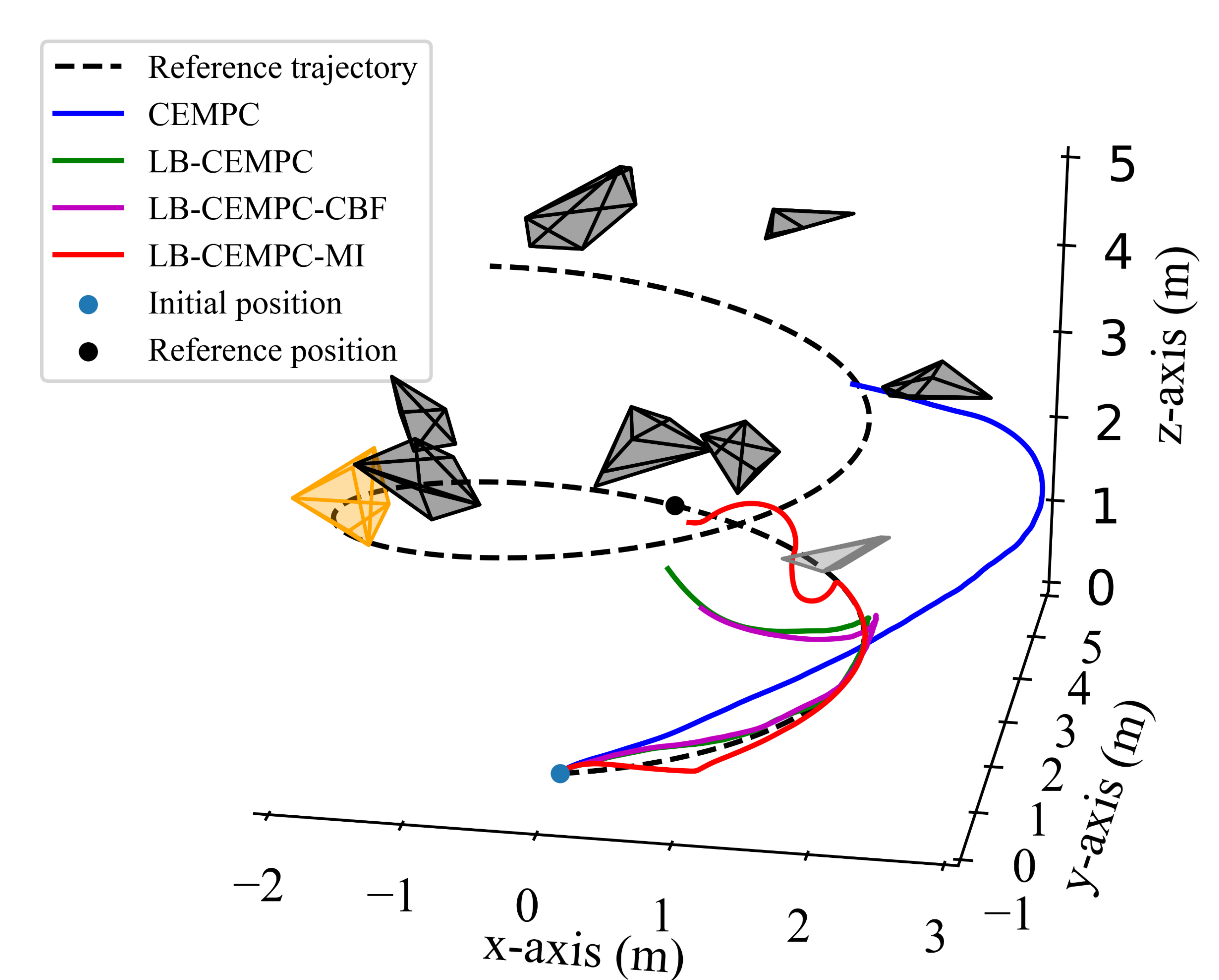}
		}\hspace{5mm}
		\subfigure[Snapshot at t = 8s]{
			\label{fig:3D_2}
			\includegraphics[scale=0.09]{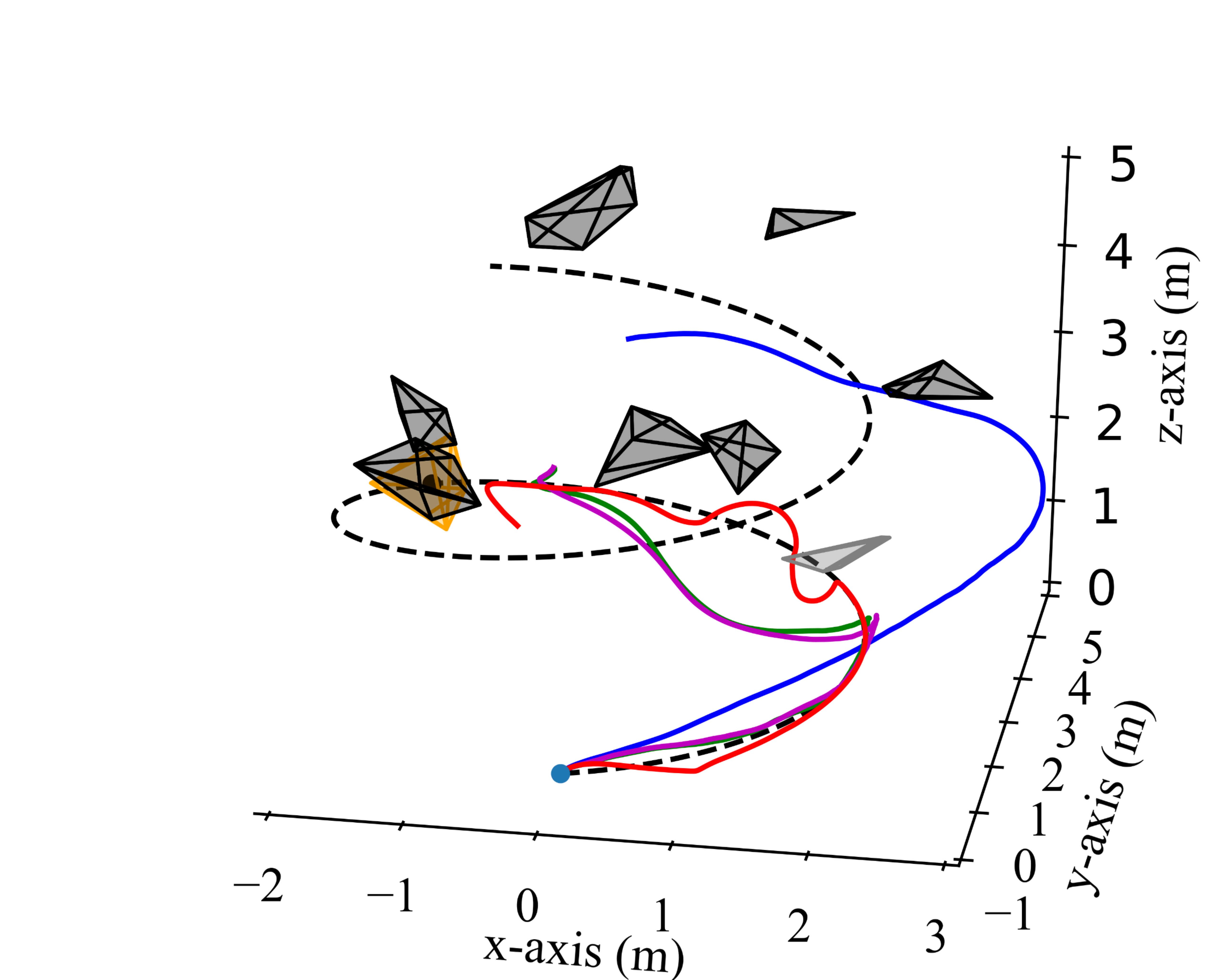}
		}
		\hspace{-5mm}		\subfigure[Snapshot at t = 12s]{
			\label{fig:3D_3}
			\includegraphics[scale=0.09]{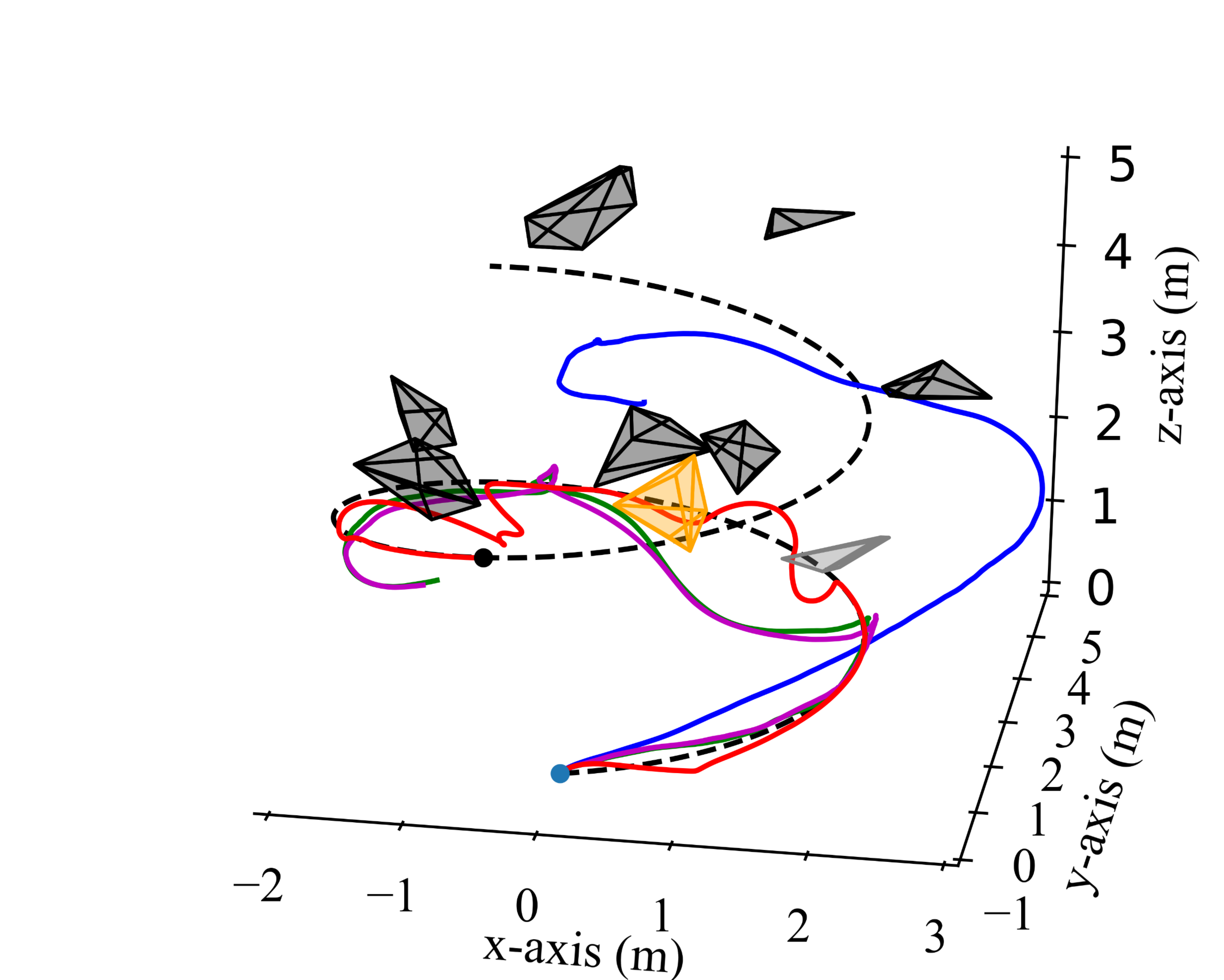}
		}\hspace{5mm}
		\subfigure[Snapshot at t = 20s]{
			\label{fig:3D_4}
			\includegraphics[scale=0.09]{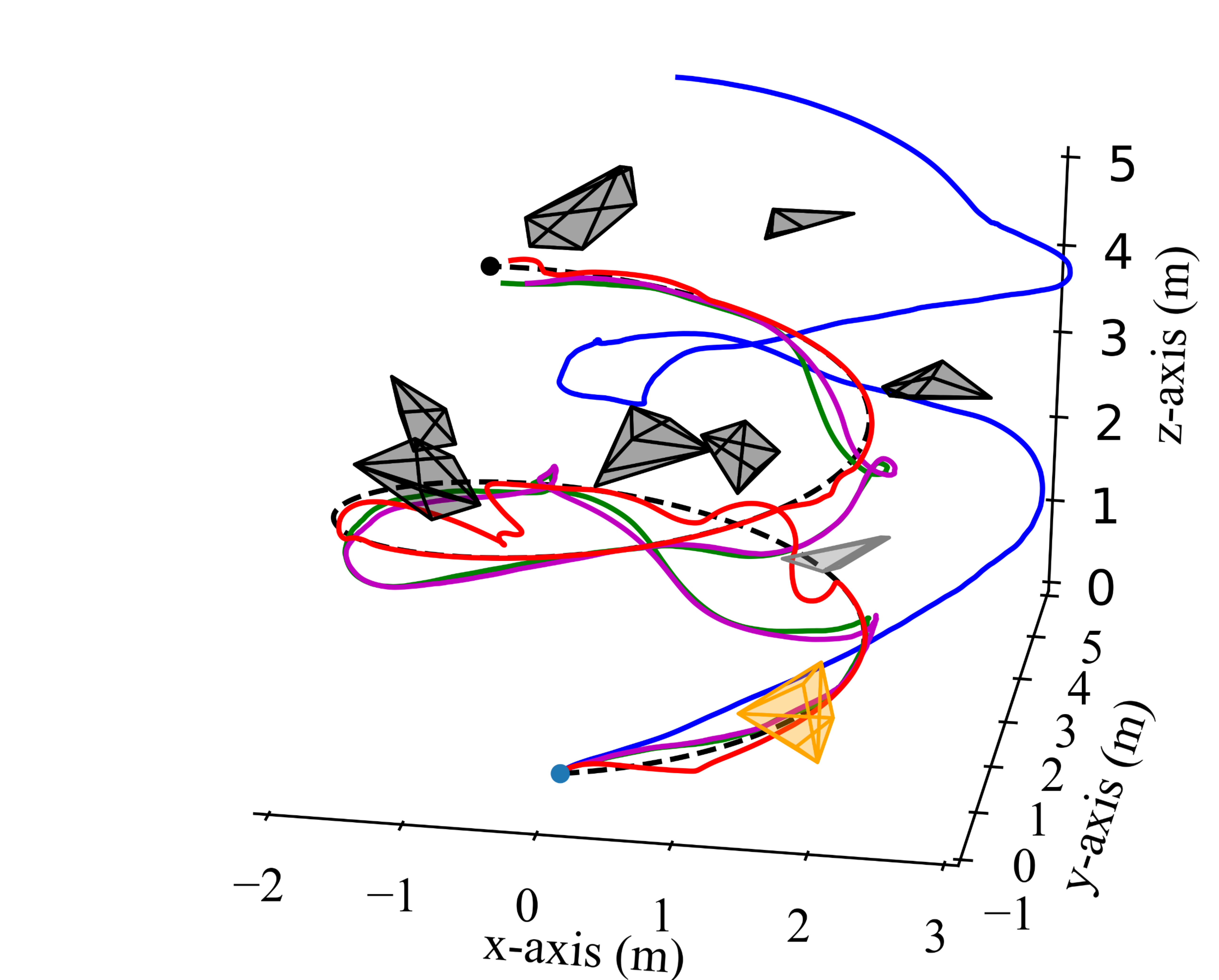}
		}\vspace{-5mm}
		\caption{Numerical validation of the trajectory tracking in \textbf{Scenario 1}, where a quadrotor flies through a densely cluttered obstacle field under uncertain Wind-4 disturbances. Snapshots of the simulation are shown in~\ref{fig:3D_1}-\ref{fig:3D_4}. The black dashed line denotes the reference trajectory. The irregular polyhedron in orange and gray denote the dynamic and static obstacles crossing the reference trajectory, respectively.}
		\label{fig:Simulation1}
	\end{figure*}

	% Moreover, we highlight two key takeaways from the results shown in Table~\ref{table:table_results}. On one hand, 
	The proposed LB-CEMPC-MI controller achieves the lowest tracking errors among the four controllers under different all four magnitudes of wind disturbances, as shown in Table~\ref{table:table_results}. Figs.~\ref{fig:Scenario1_performance},~\ref{fig:Scenario1_MI} and~\ref{fig:Simulation1} show the control performance with our proposed LB-CEMPC-MI control framework on the quadrotor under the Wind-4. At about 4.3 and 7.5 seconds, the quadrotor deviates from the reference trajectory with increasing tracking errors. It has changed its trajectory to safely avoid the collision with the grey static obstacle on the trajectory. Fig. \ref{fig:cbf1} shows that the barrier function keeps positive and safety is ensured when avoiding the obstacle. Besides, compared with the other three baseline controllers, the LB-CEMPC-MI controller can achieve smoother obstacle avoidance behavior as shown in Fig.~\ref{fig:Tracking error}. Moreover, Figs.~\ref{fig:Tracking error} and \ref{fig:Simulation1} illustrate that the MI auxiliary controller can effectively help the main task of trajectory tracking by guiding the sampling distribution. The quadrotor can be guided quickly back to the references when deviating from the desired trajectories due to the unexpected obstacle avoidance. These results demonstrate that the LB-CEMPC-MI control scheme can drive the system to track the reference trajectory stably when the trajectory is safe and can relax the tracking to avoid obstacles safely when it is to collide with the detected obstacles.

		\begin{table}[!ht]
			\renewcommand{\arraystretch}{1.1}
			\scriptsize
			\caption{Trajectory Tracking Control and Safety Performance with Different Weight Coefficients.}
			\label{table:table_predictivity}
			\centering
			\begin{tabular}{cc|cccccc}
				\hline
				$w_{i}$ & MI Scheme & 
				\begin{tabular}{@{}c@{}} Time to avoid \\ obstacle 1\\ (in second)\end{tabular}  & 
				\begin{tabular}{@{}c@{}} Time to avoid\\  obstacle 2\\ (in second)\end{tabular}  & 
				\begin{tabular}{@{}c@{}} Minimum Barrier\\ Value (in meter)\end{tabular}  &
				\begin{tabular}{@{}c@{}} Max tracking\\ error (in meter) \end{tabular}& 
				\begin{tabular}{@{}c@{}} RMS error \\ (in meter) \end{tabular}&
				\begin{tabular}{@{}c@{}} {Collide} \\ {or not} \end{tabular}\\
				\hline
				{10}                           & {-}                           & {3.283} & {11.010} & {0.990} & {1.952} & {0.630} & {no}  \\
{1}                            & {-}                           & {-}     & {-}      & {-}     & {-}     & {-}     & {yes} \\
{0.1}                          & {-}                           & {-}     & {-}      & {-}     & {-}     & {-}     & {yes} \\
{0.01}                         & {-}                           & {-}     & {-}      & {-}     & {-}     & {-}     & {yes} \\
{0}                            & {-}                           & {-}     & {-}      & {-}     & {-}     & {-}     & {yes} \\
				10      & CBF   & 3.283 & 11.010 & 0.994 & 1.957 & 0.632 & {no} \\
				1       & CBF   & 3.737 & 11.136 & 0.427 & 1.321 & 0.311 & {no} \\
				0.1     & CBF   & 3.737 & 11.364 & 0.617 & 0.772 & 0.214 & {no} \\
				0.01    & CBF   & 4.289 & 11.710 & 0.485 & 0.703 & 0.167 & {no} \\
				0       & CBF   & 4.293 & 11.717 & 0.481 & 0.723 & 0.190 & {no} \\
				\hline
			\end{tabular}
		\end{table}
		
	\begin{figure*}[!ht]
		\centering
		\subfigure[Snapshot at t = 6s.]{
			\label{fig:3D_21}
			\includegraphics[scale=0.09]{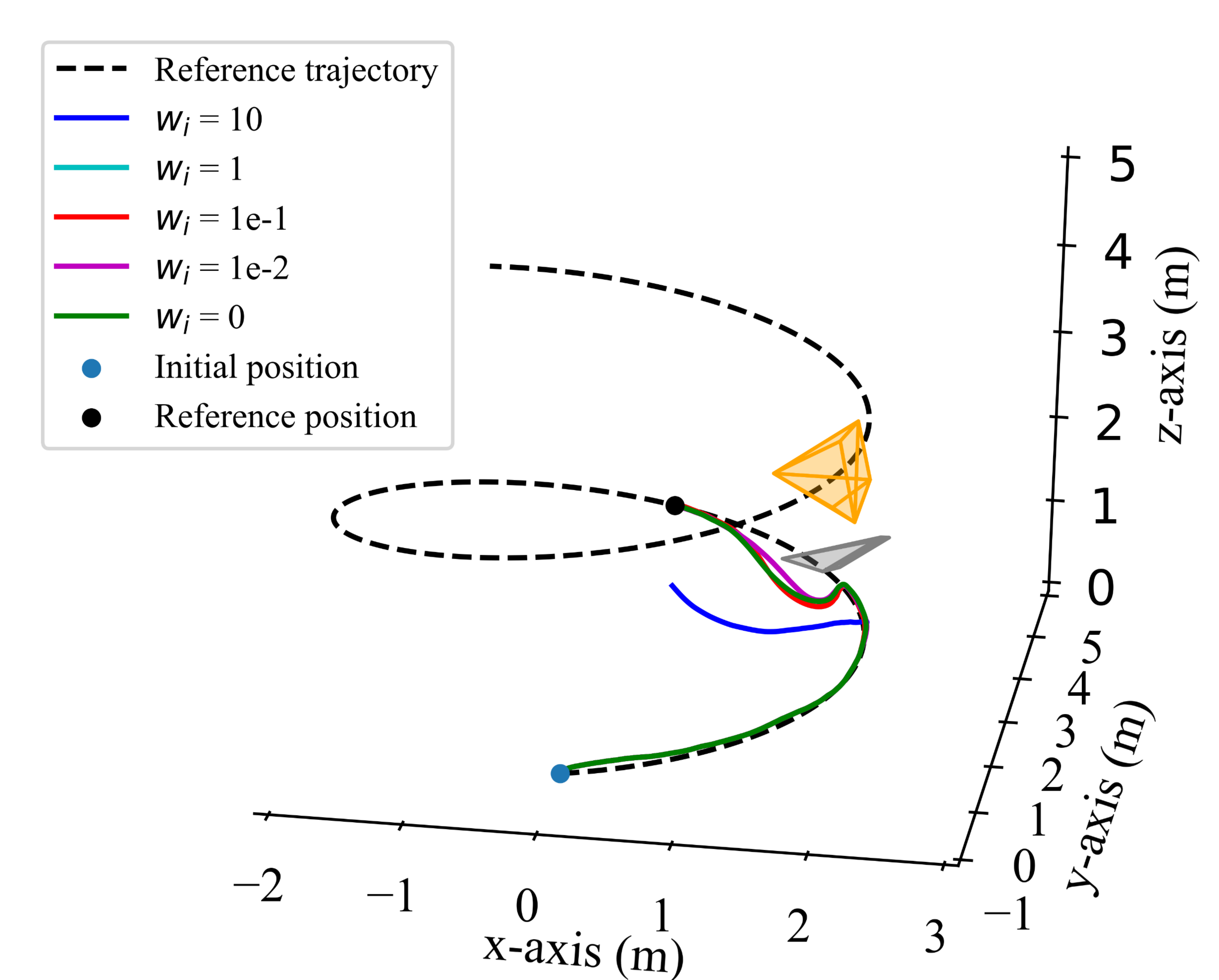}
		}\hspace{5mm}
		\subfigure[Snapshot at t = 12s]{
			\label{fig:3D_22}
			\includegraphics[scale=0.09]{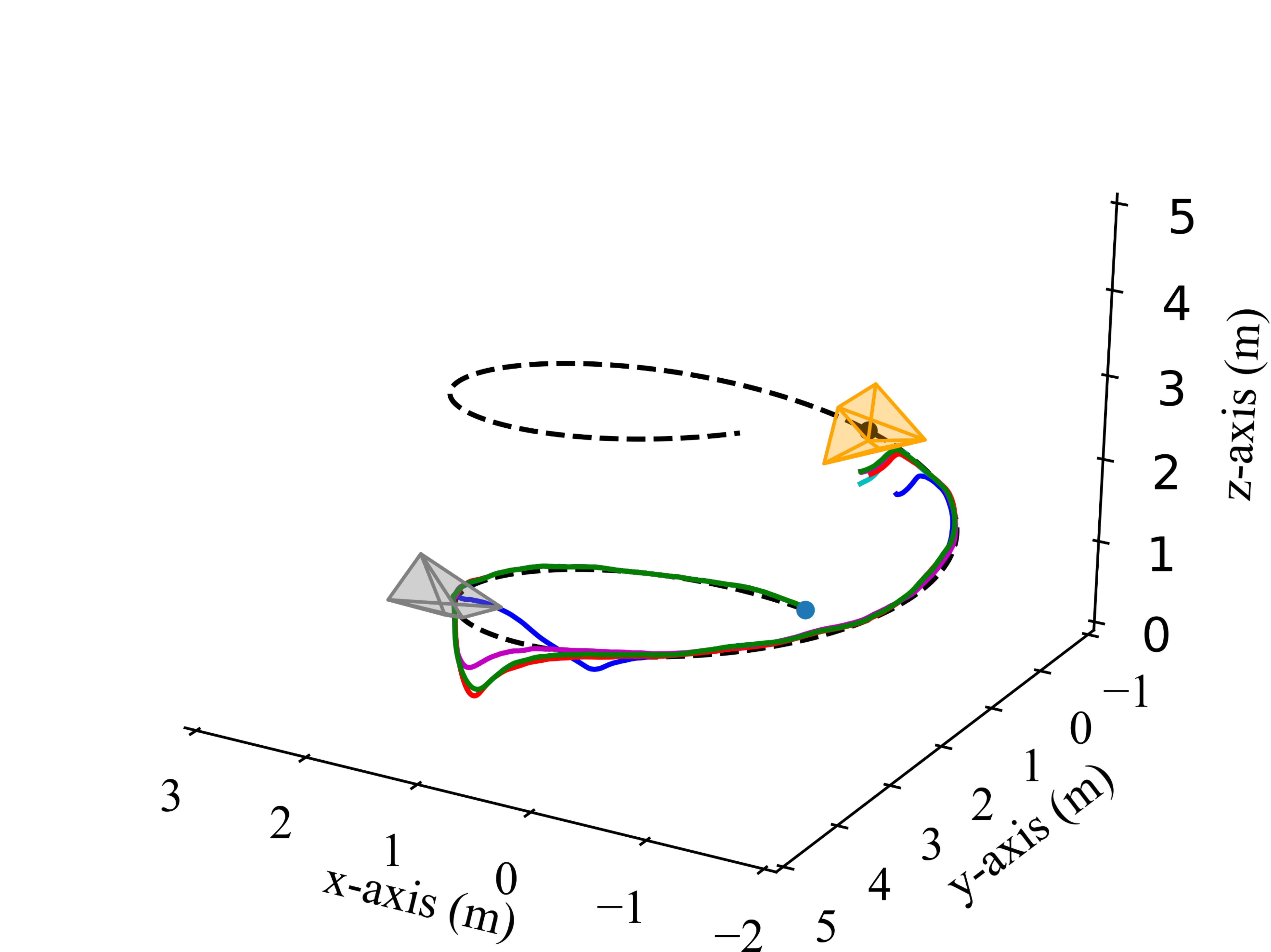}
		}\vspace{-2mm}
		\subfigure[Snapshot at t = 16s]{
			\label{fig:3D_23}
			\includegraphics[scale=0.09]{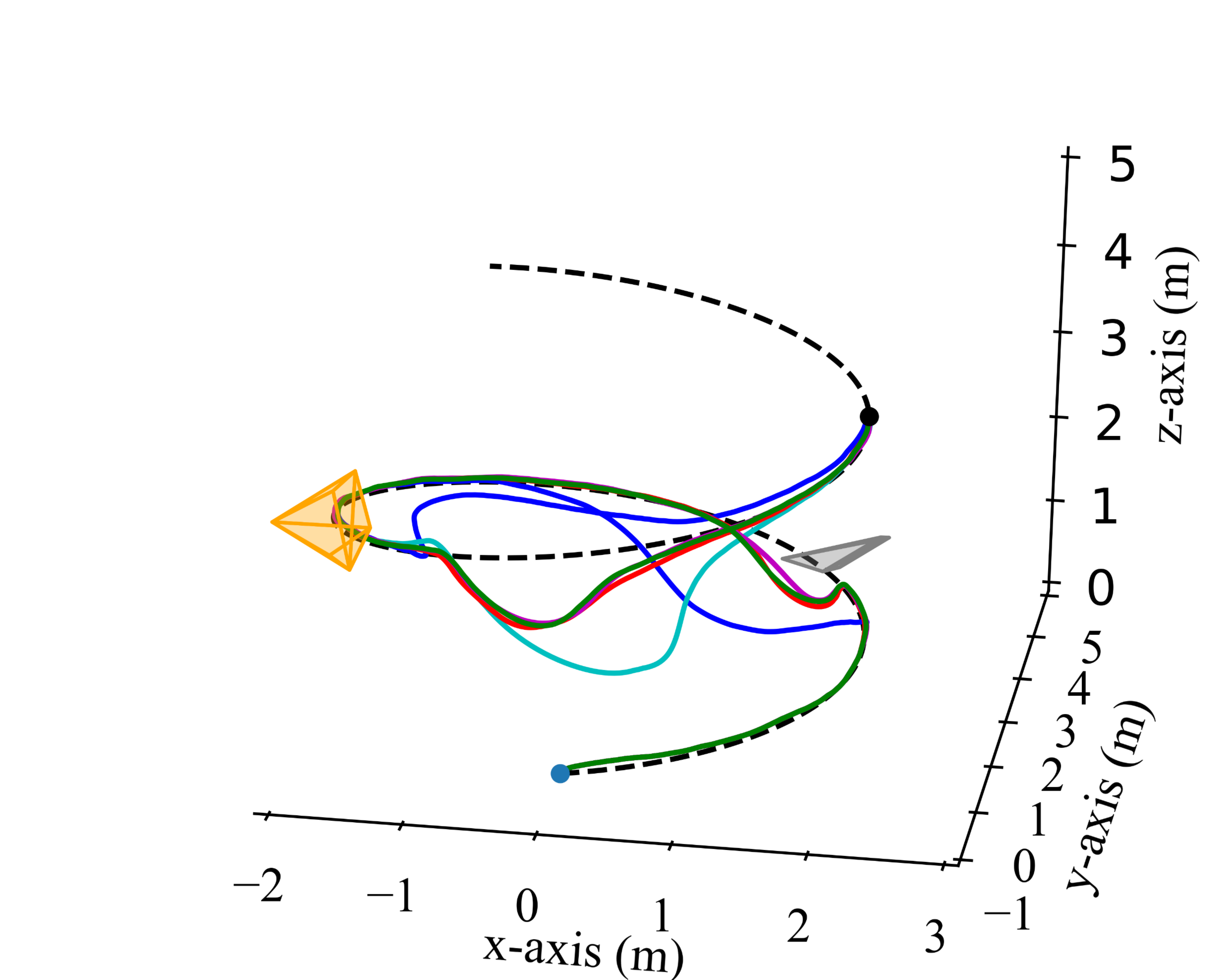}
		}\hspace{5mm}
		\subfigure[Snapshot at t = 20s]{
			\label{fig:3D_24}
			\includegraphics[scale=0.09]{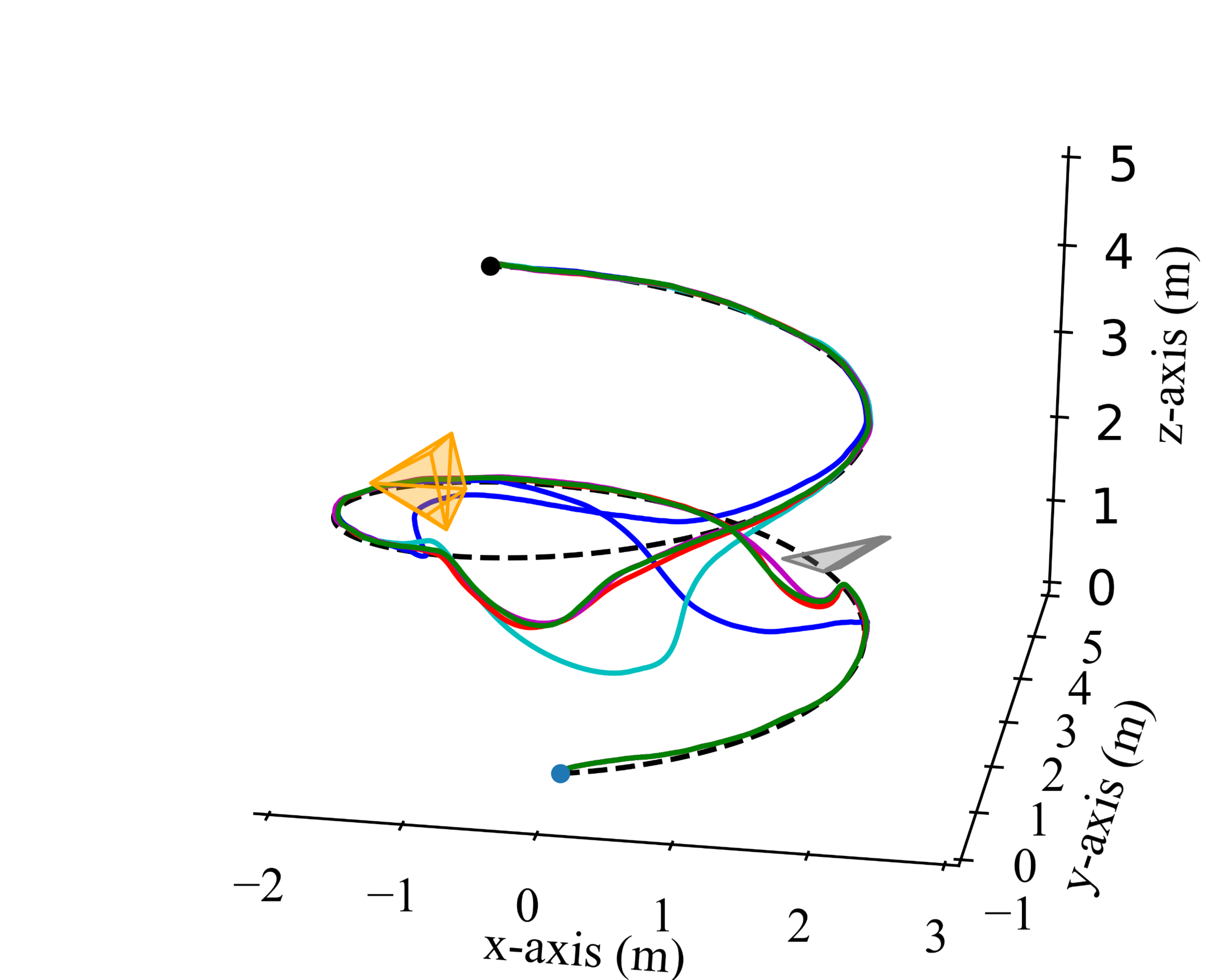}
		}
		\caption{Numerical validation of the trajectory tracking in the \textbf{Scenario 2}, where a quadrotor tracks the reference spiral trajectory under uncertain Wind-2 disturbances. Snapshots of the simulation are shown in~\ref{fig:3D_21}-\ref{fig:3D_24}. The black dashed line denotes the reference trajectory $p_d(t)$. The irregular polyhedron in orange and gray denote the dynamic and static obstacles crossing the reference trajectory, respectively. The quadrotor tracks the reference trajectory and strictly guarantees non-collision with the obstacles.}
		\label{fig:Simulation2}
	\end{figure*}
	
	\begin{figure}[!ht]
		\centering
		\subfigure[]{
			\label{fig:position_error_cbf}
			\includegraphics[scale=0.045]{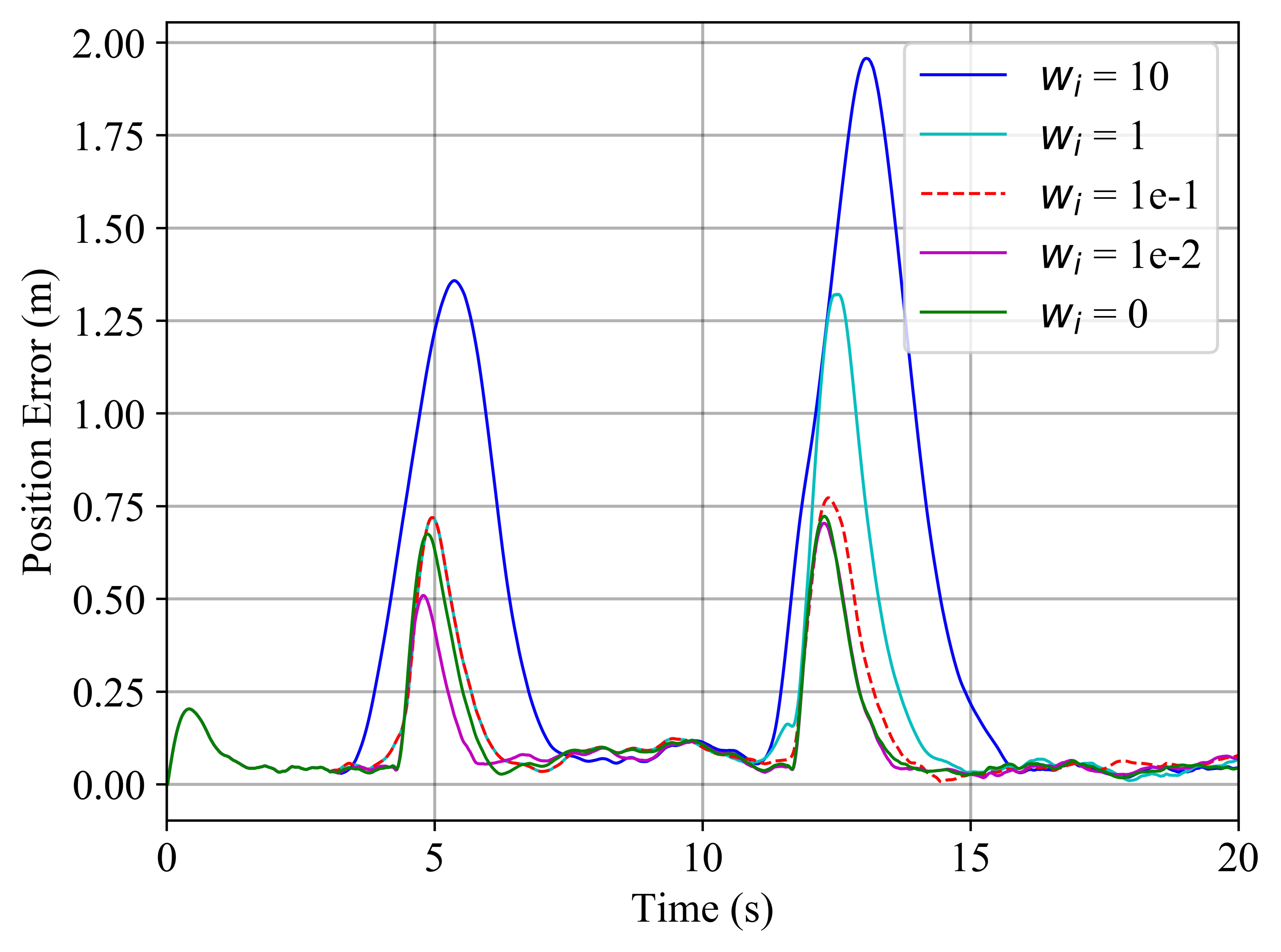}
		}
		\subfigure[]{
			\label{fig:bfv_w1}
			\includegraphics[scale=0.045]{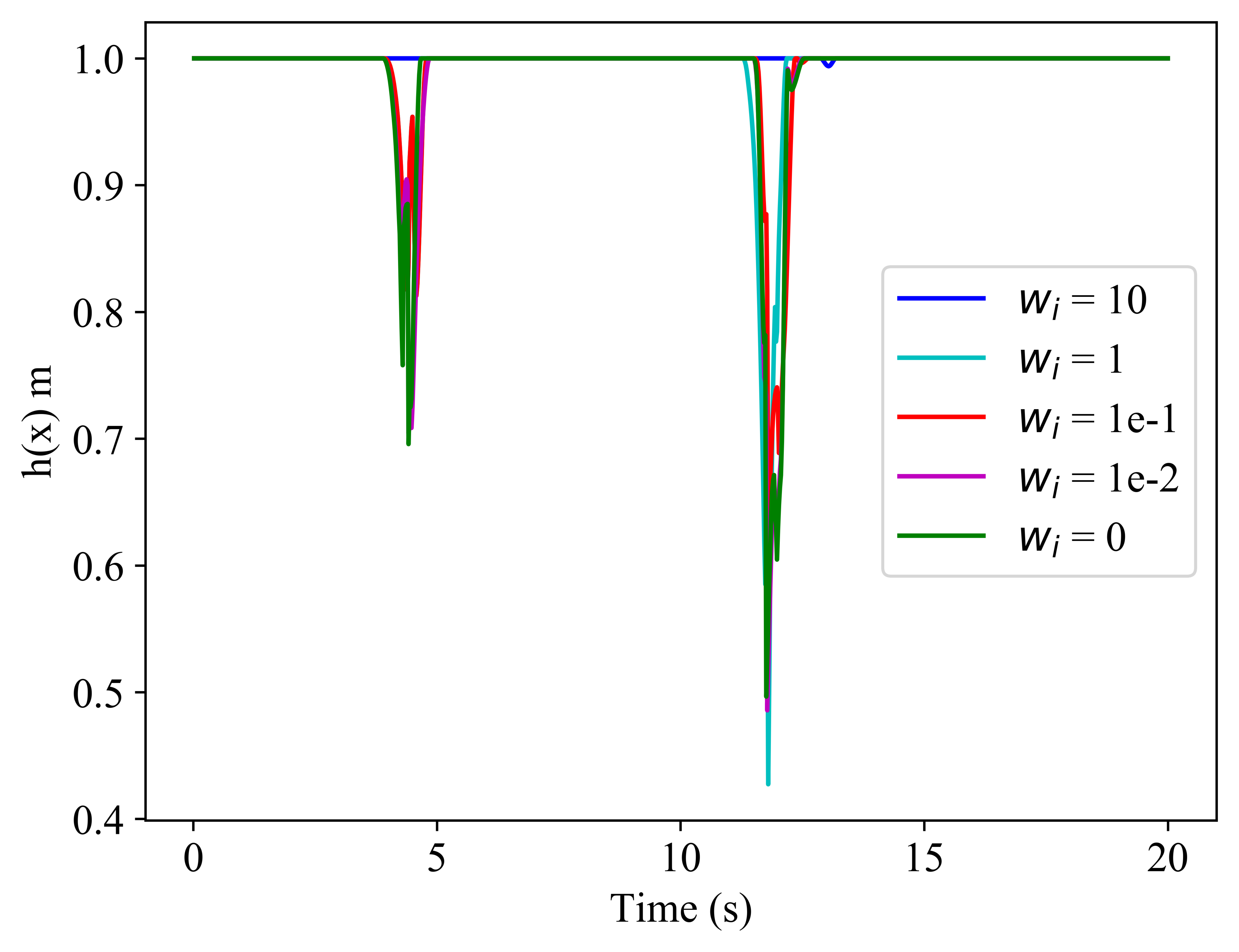}
		}
		\caption{Simulation results in \textbf{Scenario 2}. The prediction horizon is $T_h = 0.6s$. (a) The evolution of tracking errors with different weight coefficients $w_{1}$, where the dashed red line is used to distinguish the tracking error using the controller with $w_{1}=1e-1$; (b) The values of the CBF.}
		\label{fig:tracking_cbf}
	\end{figure}  
	\subsubsection{Trade-off Between Safety and Tracking Performance}
	\label{trade-off between safety and stability}
	To validate the trade-off between safety and tracking performance, \textbf{Scenario 2} with Wind-2 is fixed in this part. To demonstrate the effects of predictivity inherent in the MPC and the CBF-enforced safety scheme, we only keep the CBF (\ref{cbf}) in the MI scheme without the CLF-based sampling guidance scheme and design the LB-CEMPC algorithm with different coefficients $w_{i}$ in (\ref{cost_func}) for obstacle avoidance. {Controllers without MI scheme are also compared with controllers with CBF scheme, as shown in Table \ref{table:table_predictivity}.}
	
	% 	in a minimal modification in Algorithm \textbf{\ref{alg:Framwork_2}} for safety consideration. 
	
	% 	The LB-CEMPC with a prediction horizon $T_h = 0.6s$

	Statistic of the simulations is listed in Table \ref{table:table_predictivity} with the trajectories graphically shown in Fig.~\ref{fig:Simulation2} and tracking errors shown in Fig. \ref{fig:position_error_cbf}. We highlight four key takeaways from these results. Firstly, the controllers with $w_{i} = 10$ predict to avoid the unexpectedly gray static obstacle and the orange dynamic obstacle in advance at $3.28s$ and $11.01s$ seconds, respectively, while the controllers with $w_{i} = 0$ avoid these two obstacles at later $4.29s$ and $11.72s$. It illustrates the proposed method can keep the predictive ability for obstacle avoidance even under time-varying environmental disturbances. Secondly, the RMS and maximum tracking errors increase with a larger positive weight $w_{i}$, while the time for obstacle avoidance decreases as observed in Table \ref{table:table_predictivity}. This indicates that there is a trade-off between safety margin and tracking performance, which can be adjusted with the weight coefficient $w_{i}$ in the cost function (\ref{cost_func}). {When we need a smaller $w_i$ for better tracking performance, the CBF scheme is necessary to avoid collisions. Thirdly, when safety constraints are already satisfied, that is, $\omega_i$ = 10, it can be seen from Table \ref{table:table_predictivity} that the time to avoid the static and dynamic obstacles are the same whether the CBF scheme exists or not. And the max tracking error and RMS error are also similar. This reflects the minimal intervention of the MI scheme.} Besides, to obtain an intuitive view on the safety performance under disturbances, Fig. \ref{fig:bfv_w1} shows the safety performance of the quadrotor with CBF scheme tracking the reference trajectory. It can be seen that the values of the CBF keep positive, which indicates that the position of the quadrotor always stays within the safe obstacle-free ellipsoid region. 
	%{On the contrary, as $\omega_i$ decreases, the quadrotor without CBF scheme collides with the obstacles, as shown in Table \ref{table:table_predictivity}.}
	
	\subsubsection{Real-time Performance of CEMPC}
	\begin{table}[t]
\centering
\caption{Running Time and Tracking Performance of CEMPC with Different Parameters}
\label{tab:cempc_time}
\begin{tabular}{cc|cc}
\hline
{Predictive}        & {Number of}   & {CEMPC optimization} & {RMS error} \\
{Horizon $T_h$ (s)} & {Samples $M$} & {Time (ms)}          & {(m)}       \\ \hline
{0.2}               & {100}         & {38.960}             & {1.808}     \\
{0.2}               & {200}         & {47.607}             & {1.783}     \\
{0.2}               & {300}         & {52.106}             & {1.714}     \\
{0.4}               & {100}         & {73.575}             & {1.165}     \\
{0.4}               & {200}         & {88.966}             & {1.007}     \\
{0.4}               & {300}         & {103.984}            & {0.950}     \\
{0.6}               & {100}         & {114.629}            & {0.968}     \\
{0.6}               & {200}         & {127.892}            & {0.792}     \\
{0.6}               & {300}         & {151.410}            & {0.722}     \\ \hline
\end{tabular}
\end{table}
    
\begin{figure}[t]
\centering
\includegraphics[scale=0.7]{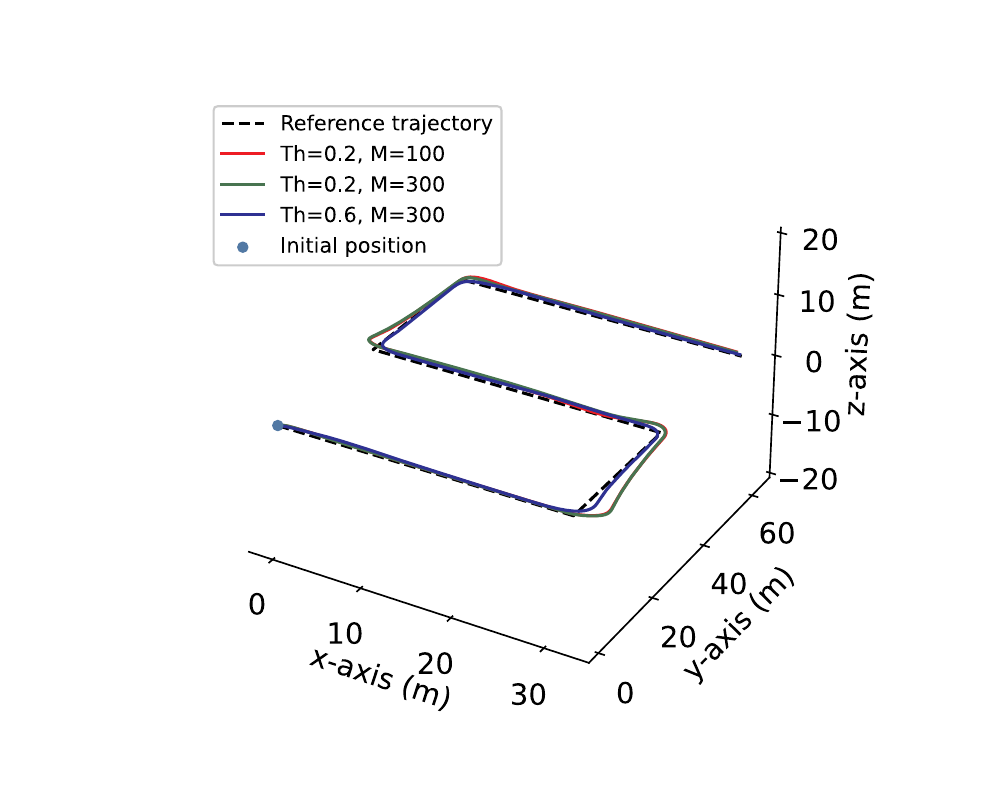}
\caption{Numerical validation of the trajectory tracking in \textbf{Scenario 3}, where quadrotors with different MPC parameters track the reference trajectory. \label{fig:S3}}
\end{figure}
    
    {The time consumption and tracking performance of CEMPC with different parameters are compared in \textbf{Scenario 3}, as shown in Table \ref{tab:cempc_time}.
    The trajectories are graphically shown in Fig.\ref{fig:S3}.
    As the iteration number is set the same, the running time increases with larger sample size or prediction horizon, which corresponds with the theoretical analysis of the algorithm complexity.
    From Table \ref{tab:cempc_time}, it can be seen that increasing the number of samples helps less for tracking performance than increasing the predictive horizon. 
    In addition, the improvement of tracking performance by increasing the predictive horizon is also limited, as increasing $T_h$ for $0.2s$ to $0.4s$ decreases the RMS errors much more than increasing $T_h$ for $0.4s$ to $0.6s$.
    Therefore, the predictive horizon is required to {be} chosen according to the circumstances and the sample size can be fixed for better real-time performance. Note that the code and computation {are} not been optimized for speed.}
	
	\subsection{Discussion}
	\label{section:Discussion}
	The predictive nature of the MPC brings proactivity to the control scheme, which improves the trajectory tracking accuracy and introduces conservation in terms of safety to the system, as shown in Section~\ref{section:results}.  
	It results in safer behaviors of the system but also degrades the control performance of the main task, e.g. trajectory tracking accuracy. Such trade-off could be considered by adjusting the weight coefficients $w_i$ in the cost function (\ref{cost_func}) according to the requirements of practical applications. For example, if larger weight coefficients $w_i$ are set to achieve better foresight to avoid obstacles, there will be a larger tracking RMS error, given a determined prediction horizon $T_h$. In contrast, if we choose small weight coefficients $w_i$ to achieve low tracking RMS error, there will be poor foresight for obstacle avoidance. Actually, it is flexible for the designer to set the weight coefficients $w_i$ according to their considerations. The designed MI scheme reserves the safety and guides the optimization, which enables the customized and convenient design of the cost function for high-level tasks.

	\section{Conclusion}
	\label{section:conclusion}
	In this paper, a safe learning-based MPC architecture is designed to optimize the nonlinear system with a {non-differentiable} objective function under uncertain environmental uncertainties. Our proposed approach allows the convenient design of objective function using simple but {non-differentiable} running-cost terms. The IGPs are utilized to estimate model uncertainties with a low computational burden and augment the prior predictive model in the MPC. Solved with the sampling-based CEM, the proposed CEMPC is augmented by an auxiliary controller based on the control Lyapunov function and the CBF to guide the sampling process and theoretically endows the system with safety in a way of minimal intervention. We provide numerical simulation results comparing the CEMPC, LB-CEMPC, LB-CEMPC-CBF, and the LB-CEMPC-MI algorithms on a quadrotor trajectory tracking and obstacle avoidance task under different wind disturbances in two different scenarios. The results show that the proposed LB-CEMPC-MI algorithm can successfully and safely optimize the quadrotor system with a conveniently designed {non-differentiable} objective function, achieving accurate tracking performance and safe obstacle avoidance under uncertain wind disturbances. 
	In future work, the proposed learning-based MPC framework will be verified in hardware platforms under realistic conditions, and extended to solve a navigation problem considering uncertain environmental disturbances.
	\medskip
	\bibliographystyle{unsrt}
 	\bibliography{mybibfile}

\newpage
\begin{appendices}
\label{app1}
\setcounter{equation}{0}
\renewcommand{\theequation}{A.\arabic{equation}}
\section{Proof of remark 5}

{
% \begin{lemma}
    Let $\sigma^2_{n+1}(x_*)$ be the predictive variance of a Gaussian process regression model at $x_*$ given a dataset of size $n+1$.
    The corresponding predictive variance using a dataset of only the first $n$ training points is denoted $\sigma^2_{n}(x_*)$.
    Then $\sigma^2_{n+1}(x_*) \leq \sigma^2_{n}(x_*)$.
    % \end{lemma}
    }
    \begin{proof}
    {
    Firstly we have
\def\theequation{A.\arabic{equation}}
\setcounter{equation}{0}
\begin{align}
        \sigma_n^{2}(x_{*})&=k(x_{*},x_{*})-\mathbf{k}_{n,*}^\intercal(K_{\sigma,n}+\sigma_{noise}^{2}I)^{-1}\mathbf{k}_{n,*},  \\
        \sigma_{n+1}^{2}(x_{*})&=k(x_{*},x_{*})-\mathbf{k}_{n+1,*}^\intercal(K_{\sigma,n+1}+\sigma_{noise}^{2}I)^{-1}\mathbf{k}_{n+1,*}
\end{align}
    where
\begin{align}
        K_{\sigma,n+1} &= 
\begin{bmatrix}
K_{\sigma,n}+\sigma_{noise}^{2}I & {\gamma} \\  {\gamma}^\intercal & k(x_{n+1},x_{n+1})
\end{bmatrix}
, \notag \\ \mathbf{k}_{n,*}^\intercal&=[k(x_1,x_*),k(x_1,x_*), \ldots,k(x_{n},x_*)]^\intercal, \notag \\ \mathbf{k}_{n+1,*}^\intercal&=[k(x_1,x_*),k(x_1,x_*), \ldots,k(x_{n+1},x_*)]^\intercal, \notag \\ {\gamma}^\intercal&=[k(x_1,x_{n+1}),k(x_1,x_{n+1}), \ldots,k(x_{n+1},x_{n+1})]^\intercal.   \notag
\end{align}
}
    
For simplicity, let $K_{n}=K_{\sigma,n}+\sigma_{noise}^{2}I$, $K_{n+1}=K_{\sigma,n+1}+\sigma_{noise}^{2}I$, $b = k(x_{n+1},x_{n+1})$ and $e = k(x_{n+1},x_*)$.
 Then we have $K_{n+1} = 
\begin{bmatrix}
K_{n} & {\gamma} \\  {\gamma}^\intercal & b
\end{bmatrix}
$, $\mathbf{k}_{n+1,*}=
\begin{bmatrix}
\mathbf{k}_{n,*} \\ e
\end{bmatrix}
$ and
\begin{align}
        \sigma_n^{2}(x_{*})&=k(x_{*},x_{*})-\mathbf{k}_{n,*}^\intercal K_n^{-1}\mathbf{k}_{n,*},  \\
        \label{eq:sigma_n+1_2}
        \sigma_{n+1}^{2}(x_{*})&=k(x_{*},x_{*})-
\begin{bmatrix}
\mathbf{k}_{n,*}^\intercal & e
\end{bmatrix}
\begin{bmatrix}
K_{n} & {\gamma} \\  {\gamma}^\intercal & b
\end{bmatrix}
^{-1}
\begin{bmatrix}
\mathbf{k}_{n,*} \\ e
\end{bmatrix}
.
\end{align}
    Notice that $K_n^{-1}, K_{n+1}^{-1}$ are symmetric as $K_n$ and $K_{n+1}$ are symmetric, the inversion of the partitioned matrix in \ref{eq:sigma_n+1_2} is
\begin{equation}
        K_{n+1}^{-1}
        = 
\begin{bmatrix}
K_{n} & {\gamma} \\  {\gamma}^\intercal & b
\end{bmatrix}
^{-1}
        = 
\begin{bmatrix}
\tilde{K_{n}} & \tilde{{\gamma}} \\  \tilde{{\gamma}}^\intercal & \tilde{b}
\end{bmatrix}
        = 
\begin{bmatrix}
K_{n}^{-1} + K_{n}^{-1}{\gamma}\tilde{b}{\gamma}^\intercal K_{n}^{-1} & -K_{n}^{-1}{\gamma}\tilde{b} \\  -\tilde{b}{\gamma}^\intercal K_{n}^{-1} & \tilde{b}
\end{bmatrix}
,
\end{equation}
    where $\tilde{b} = (b-{\gamma}^\intercal K_{n}^{-1} {\gamma})^{-1}$. It is obvious that $\tilde{b} > 0$ as $
\begin{bmatrix}
K_{n} & {\gamma} \\  {\gamma}^\intercal & b
\end{bmatrix}
$ and $K_n$ are positive definite.
    Then
\begin{align}
        \sigma_n^{2}(x_{*}) - \sigma_{n+1}^{2}(x_{*})
        &= 
\begin{bmatrix}
\mathbf{k}_{n,*}^\intercal & e
\end{bmatrix}
\begin{bmatrix}
K_{n} & {\gamma} \\  {\gamma}^\intercal & b
\end{bmatrix}
^{-1}
\begin{bmatrix}
\mathbf{k}_{n,*} \\ e
\end{bmatrix}
 - \mathbf{k}_{n,*}^\intercal K_n^{-1}\mathbf{k}_{n,*} \notag  \\
        &= 
\begin{bmatrix}
\mathbf{k}_{n,*}^\intercal & e
\end{bmatrix}
\begin{bmatrix}
\tilde{K_{n}} & \tilde{{\gamma}} \\  \tilde{{\gamma}}^\intercal & \tilde{b}
\end{bmatrix}
\begin{bmatrix}
\mathbf{k}_{n,*} \\ e
\end{bmatrix}
 - \mathbf{k}_{n,*}^\intercal K_n^{-1}\mathbf{k}_{n,*} \notag  \\
        &= \mathbf{k}_{n,*}^\intercal\tilde{K_n} \mathbf{k}_{n,*} + \mathbf{k}_{n,*}^\intercal\tilde{{\gamma}}e + e\tilde{{\gamma}}^\intercal \mathbf{k}_{n,*} + e\tilde{b}e - \mathbf{k}_{n,*}^\intercal K_n^{-1}\mathbf{k}_{n,*}  \notag  \\
        &= \mathbf{k}_{n,*}^\intercal(K_{n}^{-1} + K_{n}^{-1}{\gamma}\tilde{b}{\gamma}^\intercal K_{n}^{-1}) \mathbf{k}_{n,*} - \mathbf{k}_{n,*}^\intercal K_{n}^{-1}{\gamma}\tilde{b}e - e\tilde{b}{\gamma}^\intercal K_{n}^{-1} \mathbf{k}_{n,*}+ e\tilde{b}e - \mathbf{k}_{n,*}^\intercal K_n^{-1}\mathbf{k}_{n,*}  \notag  \\
        &= \mathbf{k}_{n,*}^\intercal K_{n}^{-1}{\gamma}\tilde{b}{\gamma}^\intercal K_{n}^{-1} \mathbf{k}_{n,*} - \mathbf{k}_{n,*}^\intercal K_{n}^{-1}{\gamma}\tilde{b}e - e\tilde{b}{\gamma}^\intercal K_{n}^{-1} \mathbf{k}_{n,*} + e\tilde{b}e  \notag  \\
        &= \tilde{b}(\mathbf{k}_{n,*}^\intercal K_{n}^{-1}{\gamma} - e)({\gamma}^\intercal K_{n}^{-1} \mathbf{k}_{n,*} - e)  \notag  \\
        &= \tilde{b}(\mathbf{k}_{n,*}^\intercal K_{n}^{-1}{\gamma} - e)^2    \notag  \\
        &\geq 0.
\end{align}
    Then we have
\begin{equation}
        \sigma^2_{n+1}(x_*) \leq \sigma^2_{n}(x_*).\qed
\end{equation}

\end{proof}
\end{appendices}

\end{document}